\documentclass[twoside]{article}
\usepackage{graphicx}
\usepackage{bbold}
\usepackage{optidef}
\usepackage{multirow}

\usepackage{amsthm}
\usepackage{amsfonts,amssymb}
\theoremstyle{definition}

\theoremstyle{plain}
\newtheorem{thm}{Theorem}
\newtheorem{lem}[thm]{Lemma}
\newtheorem{prop}[thm]{Proposition}
\newtheorem*{cor}{Corollary}
\usepackage[utf8]{inputenc}
\usepackage[T1]{fontenc}
\usepackage{optidef}
\usepackage{fixmath}
\usepackage{hyperref}
\usepackage{nicematrix}
\usepackage{xfrac} 
\usepackage{caption}
\usepackage{subcaption}
\usepackage{comment}
\usepackage{cancel}
\usepackage[ruled,linesnumbered]{algorithm2e}

\SetCommentSty{mycommfont}
\hypersetup{
    colorlinks=true,
    linkcolor=blue,
    filecolor=magenta,      
    urlcolor=black,
}

\usepackage[utf8]{inputenc} 
\usepackage[T1]{fontenc}    
\usepackage{hyperref}       
\usepackage{url}            
\usepackage{booktabs}       
\usepackage{amsmath}
\usepackage{amsfonts}       
\usepackage{nicefrac}       
\usepackage{microtype}      
\usepackage{lipsum}		
\usepackage{graphicx}
\usepackage{doi}

\newcommand{\argmin}{\mathop{\rm arg~min}\limits}
\usepackage[accepted]{aistats2025}
\usepackage[round]{natbib}

\begin{document}

%

%

\twocolumn[

\aistatstitle{Wasserstein Gradient Flow over Variational Parameter Space for Variational Inference}

\aistatsauthor{ Dai Hai Nguyen \And Tetsuya Sakurai \And  Hiroshi Mamitsuka }

\aistatsaddress{ Hokkaido University \And  University of Tsukuba \And Kyoto University } ]

\begin{abstract}
  Variational Inference (VI) optimizes variational parameters to closely align a variational distribution with the true posterior, being approached through vanilla gradient descent in black-box VI or natural-gradient descent in natural-gradient VI. In this work, we reframe VI as the optimization of an objective that concerns probability distributions defined over a \textit{variational parameter space}. Subsequently, we propose Wasserstein gradient descent for solving this optimization, where black-box VI and natural-gradient VI can be interpreted as special cases of the proposed Wasserstein gradient descent. To enhance the efficiency of optimization, we develop practical methods for numerically solving the discrete gradient flows. We validate the effectiveness of the proposed methods through experiments on synthetic and real-world datasets, supplemented by theoretical analyses.
\end{abstract}

\section{Introduction}
Many machine learning problems involve the challenge of approximating an intractable target distribution, which might only be known up to a normalization constant. Bayesian inference is a typical example, where the intractable and unnormalized target distribution is a result of the product of the prior and likelihood functions \citep{lindley1972bayesian,von2011bayesian}. Variational Inference (VI), a widely employed approach across various application domains, seeks to approximate this intractable target distribution by utilizing a variational distribution \citep{blei2017variational,jordan1999introduction}. VI is typically formulated as an optimization problem, with the objective of maximizing the evidence lower bound objective (ELBO), which is equivalent to minimizing the Kullback-Leiber (KL) divergence between the variational and target distributions. 

The conventional method for maximizing the ELBO involves the use of gradient descent, such as black-box VI (BBVI) \citep{ranganath2014black}. The gradient of the ELBO can be expressed as an expectation over the variational distribution, which is typically estimated by Monte Carlo samples from this distribution. Natural-gradient-based methods, such as natural-gradient VI (NGVI) \citep{khan2018fast}  has demonstrated its superior efficiency compared to the standard gradient descent for VI. The natural-gradient \citep{amari1998natural} can be obtained from the vanilla gradient by preconditioning it with the inverse of the Fisher information matrix (FIM). However, explicitly computing this inverse FIM is expensive in general cases. An interesting fact highlighted by \cite{khan2018fast} is that the natural-gradient concerning the natural parameters of an exponential family distribution (e.g., Gaussian) is equivalent to the standard gradient concerning the expectation parameters. This equivalence simplifies the updates and often leads to faster convergence compared to gradient-based methods. Nevertheless, the natural-gradient methods generally do not accept simple updates when dealing with mixture models such as a Gaussian mixture. To overcome this problem, \cite{lin2019fast,gunawan2024flexible} extend NGVI to mixture models which are more appropriate for complex and multi-modal posterior distributions.

Our work is motivated by the question: how can we extend gradient-based optimization methods for VI, such as BBVI and NGVI, to the cases where the variational distribution is a mixture of distributions (e.g., a Gaussian mixture)? Unlike the aforementioned methods that directly optimize for the variational parameters in VI, our approach imposes a mixing distribution over the variational parameters and optimizes this distribution using Wasserstein gradient flows (WGFs) \citep{jordan1998variational}. In our approach, we can reframe VI as the optimization of an objective function related to the mixing distribution. Then, we propose a preconditioned WGF over the space of variational parameters, using any quadratic form 
 as the distance matrix, which can be a user-defined preconditioning matrix. Prior to our work, mixture models were handled by Wasserstein variational inference (WVI), which defines a WGF over the space of means and covariance matrices of Gaussian distributions, endowed with the Bures-Wasserstein distance \citep{bhatia2019bures}.

 In summary, our approach offers the following advantages:

First, we provide a \textit{unified perspective on BBVI and NGVI}, showing that both updates can be precisely derived as particle approximation of our proposed WGFs over the variational parameter space (as shown in (\ref{eqn:particleupdate})), particularly when the number of particles is set to one. Leveraging well-established theories of WGF, we can \textit{establish theoretical insights} into the proposed methods, which can deepen our understanding of behaviors of WGFs for VI. Additionally, by using multiple particles, each representing the variational parameters of a component, we can extend BBVI and NGVI to cases where the variational distribution is a \textit{mixture of distributions}, which allow to improve the approximation of complex and multi-modal posterior distributions.

Second, our approach offers \textit{more flexibility} than WVI due to the specification of variational parameters and preconditioning matrices to induce more efficient gradient flows in the variational parameter space. Specifically, we introduce two methods, GFlowVI and NGFlowVI, which perform better than WVI in experiments on both synthetic and real-world datasets. Furthermore, we also propose an \textit{update formula for component weights} using mirror descent in the probability space with its theoretical analysis.

\section{Related Work}
We first review BBVI and NGVI, two commonly used gradient-based optimization methods for VI. Then, we provide background information on gradient flows on the probability distribution space. In addition, we present some relevant hierarchical variational models and highlight the key distinctions between them and our approach.

\subsection{Gradient-based Optimization for VI}
We consider the following problem setting. Let $D$ be a set of observations, $\textbf{z}$ be a latent variable and $q(\textbf{z}|\mathbold{\lambda})$ be the variational distribution with the variational parameter $\mathbold{\lambda}\in \mathbb{R}^{d}$, our goal is to approximate the true posterior $\pi(\textbf{z}|D)$ with $q$ by minimizing the negated ELBO:
\begin{equation}
    \label{eqn:negatedelbo}
    \min_{\mathbold{\lambda}}\mathcal{L}(\mathbold{\lambda})=E_{\textbf{z}\sim q(\cdot|\mathbold{\lambda})}\left[f(\textbf{z})\right]-H(q),
\end{equation}
where $f(\textbf{z})=-\log \pi(D,\textbf{z})$ and $H(q)$ is the entropy of $q$, given by: $H(q)=-\mathbb{E}_{\textbf{z}\sim q(\cdot|\mathbold{\lambda})}\left[ \log q(\textbf{z}|\mathbold{\lambda}) \right]$.

The negated ELBO can be optimized with the gradient descent algorithm, known as BBVI \citep{ranganath2014black}. To estimate the gradient of negated ELBO, we can use the reparameterization trick \citep{kingma2013auto}, which reparameterizes $q(\textbf{z}|\mathbold{\lambda})$ in terms of a surrogate random variable $\epsilon\sim p(\epsilon)$ and a deterministic function $g_{\mathbold{\lambda}}$ in such a way that sampling from $q(\textbf{z}|\mathbold{\lambda})$ is performed as follows: $\epsilon\sim p(\epsilon)$, $\textbf{z}=g_{\mathbold{\lambda}}(\epsilon)$. If $g_{\mathbold{\lambda}}$ and $p$ are continuous with respect to $\textbf{z}$ and $\epsilon$, respectively, the gradient of negated ELBO and parameter update are as follows:
\begin{align}
\label{eqn:vi}
\begin{split}
    \mathbold{\lambda}_{n+1} \leftarrow & \mathbold{\lambda}_{n} - \eta \nabla_{\mathbold{\lambda}}\mathcal{L}(\mathbold{\lambda}_{n}), \\
    \nabla_{\mathbold{\lambda}}\mathcal{L}(\mathbold{\lambda})= & \mathbb{E}_{\epsilon\sim p} \left[ 
    \nabla_{\mathbold{\lambda}}\left( f(g_{\mathbold{\lambda}}(\epsilon))+\log q(g_{\mathbold{\lambda}}(\epsilon)|\mathbold{\lambda}) \right)
    \right],
\end{split}
\end{align}
where $\eta$ is the learning rate, and  $\nabla_{\mathbold{\lambda}}\mathcal{L}$ can be estimated using Monte Carlo samples from  $p(\epsilon)$. 

Compared to the gradient descent, natural-gradient descent has been shown to be much more efficient for VI \citep{khan2018fast}. 
The natural-gradient descent can be obtained from the standard gradient descent by preconditioning it with the inverse Fisher Information Matrix (FIM), as follows: $
    \mathbold{\lambda}_{n+1} \leftarrow  \mathbold{\lambda}_{n} - \eta \left[\textbf{F}(\mathbold{\lambda}_{n}) \right]^{-1}\nabla_{\mathbold{\lambda}}\mathcal{L}(\mathbold{\lambda}_{n})$,
where $\textbf{F}(\mathbold{\lambda})$ is the FIM with respect to $\mathbold{\lambda}$.
However, explicitly computing the FIM can be expensive. As a more efficient alternative, \cite{khan2018fast} show that when $q$ is an exponential family distribution and $\mathbold{\lambda}$ is its natural parameter, the natural-gradient with respect to $\mathbold{\lambda}$ is equivalent to the standard gradient with respect to the expectation parameter:
\begin{align}
\label{eqn:meanparamupdate}
\begin{split}
    \mathbold{\lambda}_{n+1} \leftarrow & \mathbold{\lambda}_{n} - \eta\nabla_{\textbf{m}}\mathcal{L}(\mathbold{\lambda}_{n}),
\end{split}
\end{align}
where $\textbf{m}$ is the expectation parameter of $q$, given by: $\textbf{m}(\mathbold{\lambda})=\mathbb{E}_{\textbf{z}\sim q(\cdot| \mathbold{\lambda})}\left[ T(\textbf{z}) \right]$, where $T(\textbf{z})$ is the sufficient statistics of $q$  \citep{blei2017variational}. For many existing works on natural-gradient for VI, e.g., \cite{khan2018fast}, the above gradient is easier to compute than the gradient with respect to $\mathbold{\lambda}$ and the natural-gradient descent admits a simpler update form than gradient descent.
For instance, when $q$ is a diagonal Gaussian, i.e., $q(\textbf{z}|\mathbold{\lambda})=\mathcal{N}(\textbf{z}|\mathbold{\mu}, \mathbold{\sigma}^{2})$, the update (\ref{eqn:meanparamupdate}) becomes: 
\begin{align}
\label{eqn:ngvi}
\begin{split}
    \mathbold{\mu}_{n+1} \leftarrow & \mathbold{\mu}_{n} - \eta  \mathbold{\sigma}^{2}_{n+1}\odot\nabla_{\textbf{z}}f(\textbf{z}),\\
    \mathbold{\sigma}^{-2}_{n+1} \leftarrow & (1-\eta)\mathbold{\sigma}^{-2}_{n} + \eta \texttt{diag}\left[ \nabla^{2}_{\textbf{z}}f(\textbf{z}) \right],
\end{split}
\end{align}
where $\textbf{a}\odot\textbf{b}$ denotes the element-wise product between vectors $\textbf{a}$ and $\textbf{b}$ and $\texttt{diag}[\textbf{A}]$ denotes the function to extract diagonal entries of matrix $\textbf{A}$.

\subsection{Gradient Flows on Probability Distribution Space}
Consider the problem of minimizing $F:\mathcal{P}(\Omega)\rightarrow \mathbb{R}$, a functional in the space of probability distributions $\mathcal{P}(\Omega)$ with $\Omega \subset \mathbb{R}^{d}$. 
We first endow a Riemannian geometry on $\mathcal{P}(\Omega)$, characterized by the second-order Wasserstein (or 2-Wasserstein) distance between two distributions:
$\mathcal{W}_{2}^{2}(\rho, \rho^\prime)= \inf_{\gamma}\left\{ \int_{\Omega\times \Omega} \lVert \textbf{x} - \textbf{x}^\prime\rVert_{2}^{2}\mathrm{d}\gamma(\textbf{x},\textbf{x}^\prime):\gamma\in \Gamma(\rho,\rho^\prime)\right\},$
where $\Gamma(\rho,\rho^\prime)$ is the set of all possible couplings with marginals $\rho$ and $\rho^\prime$.
This is an optimal transport problem, which has been shown effective for comparing probability distributions in many applications, such as  \citep{petric2019got,nguyen2021learning,nguyen2023linear,nguyen2024moreau,caluya2019gradient}. If $\rho$ is absolutely continuous with respect to the Lebesgue measure, there exists a unique optimal transport plan from $\rho$ to $\rho^\prime$, i.e., a mapping $T:\Omega\rightarrow \Omega$ pushing $\rho$ onto $\rho^\prime$ satisfying $\rho^\prime=T\# \rho$, where $T\# p$ denotes the pushforward measure of $\rho$. Then, the 2-Wasserstein distance can be equivalently reformulated as:
$\mathcal{W}_{2}^{2}(\rho, \rho^\prime)= \inf_{T} \int_{\Omega} \lVert \textbf{x} - T(\textbf{x})\rVert_{2}^{2}\mathrm{d}\rho(\textbf{x})$.

Let $\{\rho_{t}\}_{t\in[0,1]}$ be an absolutely continuous curve in $\mathcal{P}(\Omega)$ with finite second-order moments. Then, for $t\in [0,1]$, there exists a velocity field $v_{t}\in L^{2}(\rho_{t})$, where $L^{2}(\rho_{t})$ denotes the space of function $h:\Omega\rightarrow\Omega$, such that the \textit{continuity equation} $\partial_{t}\rho_{t}+\text{div}(\rho_{t}v_{t})=0$ is satisfied, where  $\text{div}$ is the divergence operator \citep{santambrogio2015optimal}. Consider two distribution $\rho_{t}$ and $\rho_{t+h}$, and let $T_{h}$ be the optimal transport map between them. Define $v_{t}(\textbf{x})$ as the discrete velocity of the particle $\textbf{x}$ at time $t$, given by $v_{t}(\textbf{x})=(T_{h}(\textbf{x})-\textbf{x})/h$ (i.e., displacement/time). It is shown that, in the limit $h\rightarrow 0$, $v_{t}$ has the following form: $v_{t}(\textbf{x})=-\nabla \delta F(\rho_{t}) (\textbf{x})$ , where $\delta F(\rho_{t})$ is the first variation of $F$ at $\rho_{t}$. By the continuity equation, we get the expression of WGF of $F$, as follows:
\begin{equation}
\label{eqn:wgf}
    \partial_{t}\rho_{t}=\text{div}\left( \rho_{t}\nabla\delta F(\rho_{t})\right).
\end{equation}
Particle-based variational inference (ParVI) methods \citep{liu2016stein,liu2019understanding} use a set of particles $\left\{\textbf{x}_{k,t} \right\}_{k=1}^{K}$ to approximate $\rho_{t}$ and update the particles to approximate the WGF. Each particle $\textbf{x}_{k,t}$ is then updated as follows: 
\begin{equation}
    \label{eqn:positionupdate}
    d\textbf{x}_{k,t}=\Tilde{v}_{t}(\textbf{x}_{k,t})dt
\end{equation}
where $\Tilde{v}_{t}$ is an approximation of $v_{t}$ obtained from the empirical distribution $\Tilde{\rho}_{t}$. Therefore, different ParVI methods can be derived by selecting appropriate $\Tilde{v}_{t}$ and discretizing \eqref{eqn:positionupdate} using specific schemes such as the first order explicit Euler discretization \citep{liu2019understanding}.

\subsection{Hierarchical Variational Models}
Several methods are related to our approach, including SIVI \citep{yin2018semi}, SIVI-SM \citep{yu2023semi}, Particle SIVI \citep{lim2025particle}, SMI \citep{ronning2024elboing}, WVI \citep{lambert2022variational}. In this subsection, we highlight the distinctions between these methods and ours. 

SIVI, SIVI-SM, Particle SIVI, SMI and our methods are based on hierarchical variational models, where the variational distribution is defined as a mixture model. Both SIVI \citep{yin2018semi} and our methods define $q$ as a mixture and optimize the mixing distribution, $\rho(\mathbold{\lambda})$, rather than $q$ directly. This strategy mitigates the limitations of traditional variational families. However, due to the intractability of the variational distributions' densities, SIVI either uses a surrogate ELBO (a lower bound of ELBO) or relies on costly inner-loop MCMC runs for ELBO maximization during training. To address these challenges, SIVI-SM \citep{yu2023semi} introduces score matching for training. In contrast, our methods optimize the ELBO directly using preconditioned Wasserstein gradient descent (WGD) in the variational parameter space.

Furthermore, our methods are closely related to Particle SIVI \citep{lim2025particle} and SMI \citep{ronning2024elboing}, both of which represent the variational distributions with particles, where each particle corresponds to the parameters of a component (e.g., mean and variance of a Gaussian distribution). This particle representation enhances the expressiveness of the variational distributions, enabling them to more effectively approximate complex targets compared to other particle-based variational inference methods, such as Stein variational gradient descent (SVGD) \citep{liu2016stein}. However, our methods build on the well-established theory of WGD, offering a unified perspective on BBVI and NGVI, while also providing theoretical insights into these methods for VI.

Pior to our work, VI was handled by gradient flows defined in the Brures-Wasserstein space \citep{lambert2022variational, yi2023bridging}, a subspace of the Wasserstein space consisting of Gaussian distributions. Compared to this approach, ours offers greater flexibility due to the specification of variational parameters and preconditioning matrices, which induces more efficient gradient flows in the variational parameter space.

Finally, we introduce a new update rule for component weights using mirror descent in the probability space, accompanied by its theoretical analysis (see Subsection \ref{subsection:weightupdate}). This aspect is not addressed by the previous methods, including \cite{yin2018semi, yu2023semi, lim2025particle, ronning2024elboing, lambert2022variational,yi2023bridging}.

\section{Proposed Methods}\label{sec:gradflow}
\subsection{Gradient Flows over Variational Parameter Space}

Our perspective is motivated from the key question: how to extend gradient-based optimization, such as BBVI and NGVI, to the case where the variational distribution is assumed to be a mixture of distributions (such as a Gaussian mixture). Following the key observation already made by \cite{chen2018optimal}, we identify the variational distribution with a distribution over the variational parameters. Specifically, the variational distribution $q$ now corresponds to a mixture of an infinite number of components as follows: $
q(\textbf{z})=\int_{\Omega}q(\textbf{z}|\mathbold{\lambda})d\rho(\mathbold{\lambda}), $
where $\Omega$ denotes the variational parameter space and $\rho$ is the probability distribution over $\Omega$.
As a result, we can reformulate the negated ELBO with respect to variational parameters into a distributional optimization problem with respect to $\rho$ over variational parameters as follows:
\begin{equation}
    \label{eqn:reformulatedvi}
    \min_{\rho\in \mathcal{P}(\Omega)}\mathcal{L}(\rho)=\mathbb{E}_{\mathbold{\lambda}\sim \rho}\mathbb{E}_{\textbf{z}\sim q(\cdot| \mathbold{\lambda})}\left[ f(\textbf{z})+\log q(\textbf{z})\right],
\end{equation}
where $\mathcal{P}(\Omega)$ denotes the set of distributions over variational parameters in the context of our work.

\noindent
\textbf{Remark 1}. It is noteworthy that both VI, as expressed in \eqref{eqn:negatedelbo} and our reformulated problem, as expressed in \eqref{eqn:reformulatedvi}, involve the optimization over probability measure spaces. However, the fundamental distinction lies in the definitions of the domain $\Omega$: in \eqref{eqn:vi}, the optimization variable is the variational distribution $q(\textbf{z}|\mathbold{\lambda})$, which is defined within the domain of the \textit{latent variable} $\textbf{z}$, while in \eqref{eqn:reformulatedvi}, the variable is $\rho$, which is defined within the spaces of \textit{variational parameter} $\mathbold{\lambda}$. 

The following theorem shows the first variation of $\mathcal{L}(\rho)$. This is particularly useful in formulating the gradient flows on the probability distribution space of variational parameters.

\begin{thm}
\label{thm:firstvariation} 
(First variation of $\mathcal{L}(\rho)$). The first variation of $\mathcal{L}(\rho)$ defined in  \eqref{eqn:reformulatedvi} is given by:
\begin{equation}
    \delta \mathcal{L}(\rho)(\mathbold{\lambda})= \mathbb{E}_{\textbf{z}\sim q(\cdot|\mathbold{\lambda})}\left[ f(\textbf{z})+\log \left( \mathbb{E}_{\mathbold{\lambda}^\prime\sim \rho}\left[ q(\textbf{z}|\mathbold{\lambda}^\prime) \right] \right) \right] + 1,
\end{equation}
which can be approximated using Monte Carlo samples:
\begin{equation*}
    \delta \mathcal{L}(\rho)(\mathbold{\lambda}) \approx \frac{1}{S}\sum_{i=1}^{S} \left[ f(\textbf{z}_{i}) + \log\left( \frac{1}{K}\sum_{k=1}^{K}q(\textbf{z}_{i}|\mathbold{\lambda}_{k}) \right)\right] + 1,
\end{equation*}
where $\mathbold{\lambda}_{k}\sim \rho$, $k=1,2,...,K$ and $\textbf{z}_{i}\sim q(\cdot|\mathbold{\lambda})$, $i=1,...,S$.
\end{thm}
\noindent
The proof of Theorem \ref{thm:firstvariation} can be found in Appendix \ref{theorem1}. 
Our objective is to establish gradient flows over probability distribution spaces, where the domain $\Omega$ is defined over variational parameters. The Wasserstein gradient flow is essentially a curve $\left\{ \rho_{t} \right\}_{t\in [0,1]}$ that satisfies \eqref{eqn:wgf}. In this work, we consider a preconditioned gradient flow as follows:

\begin{equation}
    \frac{\partial \mathbold{\lambda}_{t}}{\partial t}=-\textbf{C}(\mathbold{\lambda}_{t}) \nabla_{\mathbold{\lambda}}\delta \mathcal{L}(\rho_{t})(\mathbold{\lambda}_{t}),
\end{equation}
where $\textbf{C}(\mathbold{\lambda})\in \mathbb{R}^{d\times d}$ is a positive-definite \textit{preconditioning matrix}. 
Then, the dynamic of $\rho_{t}$, the probability distribution of $\mathbold{\lambda}_{t}$, is induced by the following continuity equation:
\begin{align}
    \label{eq:parameterflow}
    \frac{\partial \rho_{t}}{\partial t} + \texttt{div}(\rho_{t}\textbf{C}v_{t})=0, v_{t}=-\nabla_{\mathbold{\lambda}} \delta \mathcal{L}(\rho_{t}).
\end{align}

\noindent
\textbf{Continuous-time dynamics}. We study the dissipation of $\mathcal{L}(\rho_{t})$ along the trajectory of the flow (\ref{eq:parameterflow}), as stated in the following proposition.

\begin{prop}
\label{lemma:continuoustimedynamics} 
The dissipation of $\mathcal{L}$ along the gradient flow (\ref{eq:parameterflow}) is characterized as follows:
\begin{equation}
    \label{eqn:dynamicinequality}
    \frac{d \mathcal{L}(\rho_{t})}{d t}=-\langle v_{t},\textbf{C}v_{t} \rangle_{L^{2}(\rho_{t})},
\end{equation}
where $\langle\cdot,\cdot\rangle_{L^{2}(\rho)}$ denotes the inner product of $L^{2}(\rho)$.
\end{prop}
\noindent
The proof of Proposition \ref{lemma:continuoustimedynamics} can be found in Appendix \ref{proof:continuoustimedynamics}. Since $\textbf{C}$ is a positive-definite matrix, the right-hand side of \eqref{eqn:dynamicinequality} is non-positive. Thus Proposition \ref{lemma:continuoustimedynamics} indicates that $\mathcal{L}$ with respect to $\rho_{t}$ decreases along the gradient flow (\ref{eq:parameterflow}). The second consequence is the following corollary.
\begin{cor}
    For any $t>0$, we have:
    \begin{align*}
        \min_{0\leq s \leq t} \langle v_{s},\textbf{C}v_{s} \rangle_{L^{2}(\rho_{t})} \leq & \frac{1}{t}\int_{0}^{t}\langle v_{s},\textbf{C}v_{s} \rangle_{L^{2}(\rho_{s})}ds \\
        \leq & \frac{\mathcal{L}(\rho_{0}) - \min_{\rho\in\mathcal{P}(\Omega)} \mathcal{L}(\rho)}{t}.
    \end{align*}
\end{cor}
The corollary indicates that the gradient norm will converge to zero as $t$ goes to infinite. However, it is not guaranteed that it converges to the globally optimal solution because of the non-convexity of $\mathcal{L}$.

\noindent
\textbf{Discrete-time dynamics}. Next we study the dissipation of $\mathcal{L}$ in discrete time. We consider the following gradient descent update in the Wasserstein space applied to $\mathcal{L}$ at each iteration $n\geq 0$:
\begin{equation}
    \rho_{n+1}=(\textbf{I}-\eta \textbf{C}v_{n})_{\#}\rho_{n},
\end{equation}
where $\textbf{I}$ is the identity map. This update corresponds to a forward Euler discretization of the gradient flow \eqref{eq:parameterflow}. Let $\rho_{0}\in\mathcal{P}(\Omega)$ be the initial distribution of parameter $\mathbold{\lambda}_{0}$, i.e. $\mathbold{\lambda}_{0}\sim \rho_{0}$. For every $n>0$, $\mathbold{\lambda}_{n}\sim \rho_{n}$, we have:
\begin{equation}
\label{eqn:gradientupdate}
    \mathbold{\lambda}_{n+1} = \mathbold{\lambda}_{n} - \eta \textbf{C}(\mathbold{\lambda}_{n})v_{n}(\mathbold{\lambda}_{n}).
\end{equation}
We study the dissipation of $\mathcal{L}(\rho_{n})$ along the gradient update \eqref{eqn:gradientupdate} in the infinite number of particles regimes (where $K$ goes to infinity). We intend to obtain a descent lemma similar to Proposition \ref{lemma:continuoustimedynamics}. However, the discrete-time analysis requires more assumptions than the continuous-time analysis. Here we assume the following for all $\mathbold{\lambda}$:

\noindent
\textbf{(A1)} Assume $\exists \alpha >0$  s.t.
$\mathbb{E}_{\textbf{z}\sim q(\cdot|\mathbold{\lambda})}\lVert \nabla_{\mathbold{\lambda}}\log q(\textbf{z}|\mathbold{\lambda})\rVert^{2}_{2}\leq \alpha.$

\noindent
\textbf{(A2)} Assume $\exists \beta>0$  s.t. $
\mathbb{E}_{\textbf{z}\sim q(\cdot|\mathbold{\lambda})}\lVert \nabla^{2}_{\mathbold{\lambda}}\log q(\textbf{z}|\mathbold{\lambda})\rVert_{\texttt{op}}\leq \beta,$
where $\lVert \textbf{A} \rVert_{\texttt{op}}=\sup_{\lVert \textbf{u} \rVert_{2}=1}\lVert \textbf{A}\textbf{u} \rVert_{2}$ is the operator norm of matrix $\textbf{A}$.

\noindent
\textbf{(A3)} Assume $\exists M_{1},M_{2}>0$  s.t. $\mathbb{E}_{\textbf{z}\sim q(\cdot|\mathbold{\lambda})}|f(\textbf{z})|\leq M_{1}, \mathbb{E}_{\textbf{z}\sim q(\cdot|\mathbold{\lambda})}|\log q(\textbf{z})|\leq M_{2}.$

Assumptions \textbf{(A1)} and \textbf{(A2)} may not hold in general cases. For instance, when $q(\textbf{z}|\mathbold{\lambda})$ is a Gaussian and $\mathbold{\lambda}=\left[\mathbold{\mu}, \mathbold{\Sigma}^{-1} \right]$, the gradient and Hessian of $\log q(\textbf{z}|\mathbold{\lambda})$ with respect to $\mathbold{\Sigma}^{-1}$ cannot be bounded. However, it is possible to make \textbf{(A1)} and \textbf{(A2)} hold by imposing constraints on the covariance matrix $\mathbold{\Sigma}$. Suppose that $a\textbf{I}\preceq\mathbold{\Sigma}\preceq b\textbf{I}$ ($0<a<b$), it can be verified that by setting $\alpha_{1}=\alpha_{2}=1/4b + a^{-2}b$ and $\beta_{1}=\beta_{2}=a^{-1}$,  Assumptions \textbf{(A1)} and \textbf{(A2)} hold. We discuss how to tackle the constraints during the optimization process in Subsection \ref{subsection:fixhessian}.

 Given our assumptions, we quantify the decreasing of $\mathcal{L}$ along the gradient update \eqref{eqn:gradientupdate}, as follows.

\begin{prop}
\label{lemma:descentlemma} 
Assume \textbf{(A1)},\textbf{(A2)} and \textbf{(A3)} hold. Let $\kappa=(\alpha+\beta)(M_{1}+M_{2})$, and choose sufficiently small learning rate $\eta < 2/\kappa$. Then we have:
\begin{equation}
    \label{ineqn:descentlemma}
    \mathcal{L}(\rho_{n+1})-\mathcal{L}(\rho_{n})\leq -\eta \left(1-\kappa\frac{\eta}{2} \right)\langle v_{n},\textbf{C}v_{n} \rangle_{L^{2}(\rho_{n})}.
\end{equation}
\end{prop}

The proof of Proposition \ref{lemma:descentlemma} can be found in Appendix \ref{appendix:descentlemma}. Proposition \ref{lemma:descentlemma} indicates that the objective $\mathcal{L}(\rho_{n})$ decreases by the gradient update \eqref{eqn:gradientupdate} since the right-hand side of \eqref{ineqn:descentlemma} is non-positive by choosing a sufficiently small learning rate and the positive-definiteness of $\textbf{C}$. The following corollary is directly derived from the descent lemma.
\begin{cor}
    Let $\eta< 2/\kappa$ and $c_{\eta}=\eta\left(1-\kappa \frac{\eta}{2} \right)$. Then, we have:
    \begin{align*}
        \min_{i=1,2,...,n} \langle v_{i},\textbf{C}v_{i} \rangle_{L^{2}(\rho_{i})} \leq & \frac{1}{n}\sum_{i=1}^{n}\langle v_{i},\textbf{C}v_{i} \rangle_{L^{2}(\rho_{i})}\\
        \leq & \frac{\mathcal{L}(\rho_{0}) - \min_{\rho\in\mathcal{P}(\Omega)} \mathcal{L}(\rho)}{c_{\eta}n}.
    \end{align*}
\end{cor}
The corollary indicates that the gradient norm will converge to zero as $n$ increases. However, similar to the argument mentioned in the continuous-time analysis, it is not guaranteed that it converges to the globally optimal solution because of the non-convexity of $\mathcal{L}$.

\subsection{Weight Update via Infinite-dimensional Mirror Gradient Iterates}
\label{subsection:weightupdate}
In (\ref{eqn:gradientupdate}), only particle position, i.e. $\mathbold{\lambda}_{n}$, is updated, while its weight $\rho_{n}(\mathbold{\lambda}_{n})$ is kept fixed throughout the optimization. This weight restriction may limit the approximation capacity of $q$, especially when the number of particles is limited. To address it, we propose a scheme to update the weights of particles via the infinite-dimensional Mirror Descent (MD). 

\begin{thm}
\label{theorem:entropymd}
    (Infinite-dimensional MD) We define an iterate of infinite-dimensional Mirror Descent as follows: given  $\mu\in\mathcal{P}(\Omega)$, the learning rate $\eta$, and a function $g:\Omega\rightarrow\mathbb{R}$, we have: 
    \begin{align}
    \begin{split}
        \mu^{+}=&\text{MD}_{\eta}(\mu, g)\\
        =&\argmin_{\rho\in\mathcal{P}(\Omega)}\left\{ \eta \int_{\Omega}g(\mathbold{\lambda})(\rho(\mathbold{\lambda})-\mu(\mathbold{\lambda}))d\mathbold{\lambda}+\text{KL}(\rho, \mu)\right\},
    \end{split}
    \end{align}
    which can be equivalently defined as follows:
    for all $\mathbold{\lambda}\in\Omega$, $\mu^{+}(\mathbold{\lambda})\propto \mu(\mathbold{\lambda})\exp\left( -\eta g(\mathbold{\lambda})\right)$.
\end{thm}

The proof of Theorem \ref{theorem:entropymd} is straightforwardly extended from the celebrated entropy mirror descent in finite dimensional space (see \cite{hsieh2019finding} for more details). With the MD iterate defined above, we can define the following updates for both particle positions and their weights:
\begin{equation}
\label{eqn:mirrorWGD}
    \bar{\rho}_{n}=(\textbf{I}-\eta \textbf{C}\nabla_{\mathbold{\lambda}}\delta\mathcal{L}(\rho_{n}))_{\#}\rho_{n}, 
    \rho_{n+1}=\text{MD}_{\eta}(\bar{\rho}_{n}, \delta\mathcal{L}(\bar{\rho}_{n})),
\end{equation}
where the first update of (\ref{eqn:mirrorWGD}), corresponding to the Wasserstein transport, is responsible for updating the particles positions, i.e. $\mathbold{\lambda}$, while the second update, corresponding to the Mirror Descent part, is responsible for updating the weights, i.e. $\rho_{n}(\mathbold{\lambda})$. We show the following descent lemma for (\ref{eqn:mirrorWGD}).

\begin{prop}
\label{lemma:mirrordescentlemma} 
Assume \textbf{(A1)},\textbf{(A2)} and \textbf{(A3)} hold. Let $\kappa=(\alpha+\beta)(M_{1}+M_{2})$, and choose sufficiently small learning rate $\eta < \min\left\{2/\kappa, 1\right\}$. Then we have:
\begin{align}
\begin{split}
    \mathcal{L}(\rho_{n+1})-\mathcal{L}(\rho_{n})\leq & -\eta \left(1-\kappa\frac{\eta}{2} \right)\langle v_{n},\textbf{C}v_{n} \rangle_{L^{2}(\rho_{n})}\\
    - & \left(\frac{1}{\eta}-1 \right)KL(\rho_{n+1},\bar{\rho}_{n})
\end{split}
\end{align}
\end{prop}
The proof of Proposition \ref{lemma:mirrordescentlemma} can be found in Appendix \ref{proof:mirrordescentlemma}. Compared to Proposition \ref{lemma:descentlemma}, Proposition \ref{lemma:mirrordescentlemma} demonstrates a stronger decrease per iteration, attributed to the non-negative KL term, highlighting the advantage of incorporating MD iterates to enhance the convergence of our proposed updates.

\noindent
\textbf{Remark 2}. We emphasize that the proposed updates (\ref{eqn:mirrorWGD}) are closely related to Wasserstein-Fisher-Rao gradient flow of $\mathcal{L}$ \citep{gallouet2017jko}. Specifically, we demonstrate in Appendix \ref{WFRandMD} that the second update of (\ref{eqn:mirrorWGD}) aligns with the Fisher-Rao gradient flow as $\eta$ approaches 0. As a result, the proposed updates (\ref{eqn:mirrorWGD}) can be viewed as the discrete approximation of the \textit{preconditioned} version of Wasserstein-Fisher-Rao gradient flow of $\mathcal{L}$.

\subsection{Particle Approximation of Gradient Flows}
For solving problem \eqref{eqn:reformulatedvi} using the updates \eqref{eqn:mirrorWGD}, we assume that $\rho_{n}$ is described by a set of particles $\{ \mathbold{\lambda}_{k,n}\}_{k=1}^{K}$ and weights $\{ a_{k,n}\}_{k=1}^{K}$.
Then, the variational distribution $q_{n}(\textbf{z})$ at iteration $n$ corresponds to a familiar mixture model with a finite number of components as follows: $q_{n}(\textbf{z})=\mathbb{E}_{\mathbold{\lambda}\sim \rho_{n}} \left[ q(\textbf{z}|\mathbold{\lambda}_{k,n})\right]=\sum_{k=1}^{K} a_{k,n}q(\textbf{z}|\mathbold{\lambda}_{k,n})$. We perform the first update of \eqref{eqn:mirrorWGD} on particle positions, as follows:
\begin{equation}
    \label{eqn:particleupdate}
    \mathbold{\lambda}_{k,n+1}=\mathbold{\lambda}_{k,n}-\eta \textbf{C}(\mathbold{\lambda}_{k,n})\nabla \delta \mathcal{L}(\rho_{n})(\mathbold{\lambda}_{k,n}).
\end{equation}
We perform the second update of \eqref{eqn:mirrorWGD} on the weights of particles as follows:
\begin{equation}
    a_{k,n+1} \propto a_{k,n}\exp(-\eta \delta \mathcal{L}(\bar{\rho}_{n})(\mathbold{\lambda}_{k,n+1})),
\end{equation}
where $\bar{\rho}_{n}(\mathbold{\lambda})=\sum_{k=1}^{K}a_{k,n}\delta_{\mathbold{\lambda}_{k,n+1}}(\mathbold{\lambda})$.
We now demonstrate the simplicity of our update \eqref{eqn:particleupdate} when $q(\textbf{z}|\mathbold{\lambda})$ is a diagonal Gaussian distribution. 

\noindent
\textbf{Gradient flow VI (GFlowVI)}. First, we consider the case $\textbf{C}=\textbf{I}$. Let $\mathbold{\lambda}_{k,n}=(\mathbold{\mu}_{k,n},\textbf{s}_{k,n})$ be the $k$-th variational parameter at the $n$-th iteration, where $\textbf{s}_{k,n}=\mathbold{\sigma}^{-2}_{k,n}$ for $k =1,2,...,K$. Each variational parameter corresponds to a Gaussian distribution, and so we refer it to as a "Gaussian particle". For each iteration $n$, we generate a sample $\textbf{z}$ from $q(\textbf{z}|\mathbold{\lambda}_{k,n})$. Then we update $\mathbold{\mu}_{k,n}$ and $\textbf{s}_{k,n}$ as follows:

\begin{align}
\label{eqn:gflow-vi}
\begin{split}
    &\mathbold{\mu}_{k,n+1}= \mathbold{\mu}_{k,n}-\eta \left[\nabla_{\textbf{z}}f(\textbf{z})+\nabla_{\textbf{z}}\log q_{n}(\textbf{z})\right]\\
    &- \eta\textbf{w}_{k}\nabla_{\mathbold{\mu}_{k}}\log q(\textbf{z}|\mathbold{\lambda}_{k,n}).\\
     &\textbf{s}_{k,n+1}= \textbf{s}_{k,n} -\eta \textbf{w}_{k}\nabla_{\textbf{s}_{k}}\log q(\textbf{z}|\mathbold{\lambda}_{k,n})\\
     &+ \frac{\eta}{2} \oslash (\textbf{s}_{k,n}\odot \textbf{s}_{k,n})\odot \texttt{diag}\left[\nabla^{2}_{\textbf{z}}f(\textbf{z})+\nabla^{2}_{\textbf{z}}\log q_{n}(\textbf{z})\right]
\end{split}
\end{align}

where $\textbf{w}_{k}=q(\textbf{z}|\mathbold{\lambda}_{k,n})/q_{n}(\textbf{z})$, and $\textbf{a}\oslash\textbf{b}$ denotes the element-wise division between vectors $\textbf{a}$ and $\textbf{b}$. We refer to this update as gradient-flow VI (GFlowVI). The update (\ref{eqn:gflow-vi}) is derived using the Bonnet's and Price's theorems \citep{bonnet1964transformations,price1958useful} for the Gaussian distribution  $q(\textbf{z}|\mathbold{\lambda})=\mathcal{N}(\textbf{z}|\mathbold{\mu},\mathbold{\Sigma})$:
\begin{align}
\label{eqn:relation1}
\begin{split}
        \nabla_{\mathbold{\mu}}\mathbb{E}_{\textbf{z}\sim q(\cdot|\mathbold{\lambda})}\left[f(\textbf{z})\right]=&\mathbb{E}_{\textbf{z}\sim q(\cdot|\mathbold{\lambda})}\left[\nabla_{\textbf{z}}f(\textbf{z})\right],\\
        \nabla_{\mathbold{\Sigma}}\mathbb{E}_{\textbf{z}\sim q(\cdot|\mathbold{\lambda})}\left[f(\textbf{z})\right]=&\frac{1}{2}\mathbb{E}_{\textbf{z}\sim q(\cdot|\mathbold{\lambda})}\left[\nabla^{2}_{\textbf{z}}f(\textbf{z})\right].
\end{split}
\end{align}
\noindent
\textbf{Natural-gradient flow VI (NGFlowVI)}. Next, we consider the case that $\textbf{C}=\textbf{F}^{-1}$ is the inverse FIM. As much discussed in previous studies (e.g., \cite{khan2018fast}), the natural-gradient update does not need inverting the \textbf{FIM} for specific types of models and applications, e.g. exponential-family distributions. Thus, in this case, we consider $\mathbold{\lambda}$ to be the natural parameters of the Gaussian $q(\textbf{z}|\mathbold{\lambda})$. Specifically, the natural parameters and expectation parameters can be defined as follows:
\begin{equation*}
\mathbold{\lambda}^{(1)}_{k,n}=\textbf{s}_{k,n}\odot\mathbold{\mu}_{k,n}, \mathbold{\lambda}^{(2)}_{k,n}=-\frac{1}{2}\textbf{s}_{k,n}
\end{equation*}
and 
\begin{equation*}
\textbf{m}^{(1)}_{k,n}=\mathbold{\mu}_{k,n},\textbf{m}^{(2)}_{k,n}=\mathbold{\mu}_{k,n}\odot\mathbold{\mu}_{k,n}+1\oslash\textbf{s}_{k,n}.
\end{equation*}
The computational efficiency of the natural-gradients is a result of the following relation:
\begin{equation}
    \textbf{F}^{-1}(\mathbold{\lambda}_{k,n})\nabla_{\mathbold{\lambda}}\delta \mathcal{L}(\delta \rho_{n})(\mathbold{\lambda}_{k,n})=\nabla_{\textbf{m}}\delta \mathcal{L}(\delta \rho_{n})(\mathbold{\lambda}_{k,n}).
\end{equation}

Using the relation \eqref{eqn:relation1} and relation in \cite{khan2018fast}, the update \eqref{eqn:particleupdate} becomes:
\begin{align}
\label{eqn:ngflow-vi}
\begin{split}
    \mathbold{\mu}_{k,n+1}=& \mathbold{\mu}_{k,n}-\eta \left[\nabla_{\textbf{z}}f(\textbf{z})+\nabla_{\textbf{z}}\log q_{n}(\textbf{z})\right]\oslash\textbf{s}_{k,n+1}\\
    -& \eta \textbf{w}_{k}  \nabla_{\mathbold{\mu}_{k}}\log q(\textbf{z}|\mathbold{\lambda}_{k,n})\oslash\textbf{s}_{k,n+1}\\
     \textbf{s}_{k,n+1}=& \textbf{s}_{k,n}+\eta \texttt{diag}\left[\nabla^{2}_{\textbf{z}}f(\textbf{z})+\nabla^{2}_{\textbf{z}}\log q_{n}(\textbf{z})\right]\\
     -& 2\eta \textbf{w}_{k}(\textbf{s}_{k,n}\odot\textbf{s}_{k,n})\odot\nabla_{\textbf{s}_{k}}\log q(\textbf{z}|\lambda_{k,n})
\end{split}
\end{align}
We refer to this update as natural-gradient flow VI (NGFlowVI). Detailed derivations of the updates \eqref{eqn:gflow-vi} and \eqref{eqn:ngflow-vi} can be found in Appendix \ref{derivations}.

\noindent
\textbf{Remark 3}. Note that the gradient update (\ref{eqn:particleupdate}) can serve as a generalization of BBVI and NGVI. Indeed, when $K=1$, $\textbf{C}=\textbf{I}$ and $\mathbold{\lambda}_{n}=\left(\mathbold{\mu}_{n},\mathbold{\sigma}_{n} \right)$, the update (\ref{eqn:particleupdate}) recovers the update (\ref{eqn:vi}) of BBVI with the reparameterization trick. Also, when $K=1$, $\textbf{C}=\textbf{F}^{-1}$ and $\mathbold{\lambda}_{n}=\left(\mathbold{\sigma}^{-2}_{n}\odot\mathbold{\mu}_{n},-1/2\mathbold{\sigma}^{-2}_{n} \right)$, the update (\ref{eqn:particleupdate}) recovers the update (\ref{eqn:ngvi}) of NGVI.


\subsection{A Simple Fix to Negative Hessian Problem}
\label{subsection:fixhessian}
In the updates \eqref{eqn:gflow-vi} and \eqref{eqn:ngflow-vi}, the vectors $\textbf{s}_{k,n}$ (for $k=1,2,...,K$) are updated based on the Hessian of $f(\textbf{z})+\log q(\textbf{z})$. Since $f(\textbf{z})+\log q(\textbf{z})$ is a non-convex, its Hessian may not be positive-definite, leading to instability. Although the 
generalized Gaussian-Newton approximation from \cite{khan2018fast} is suggested to address it, it does not effectively solve the problem when applied to $f(\textbf{z})+\log q(\textbf{z})$. We introduce a solution to this issue by the approach by \cite{nguyen2023mirror}, which updates  particles within a constrained domain. We observe that in the updates \eqref{eqn:gflow-vi} and \eqref{eqn:ngflow-vi}, the variance vectors appear independently in the second part of the variational parameters, i.e. $\mathbold{\lambda}_{k,n}^{(2)}=\textbf{s}_{k,n}$ for the first case ($\textbf{C}=\textbf{I}$) and $\mathbold{\lambda}_{k,n}^{(2)}=-1/2\textbf{s}_{k,n}$ for the second case ($\textbf{C}=\textbf{F}^{-1}$). Thus, we can reformulate our problem into updating particles within the constrained domain, as addressed by \cite{nguyen2023mirror}. We define a strongly convex function $\varphi$ as follows:

\begin{equation*}
    \varphi(\mathbold{\lambda})=\frac{1}{2}\lVert \mathbold{\lambda}^{(1)} \rVert_{2}^{2}+ \langle \mathbold{\lambda}^{(2)}, \log \mathbold{\lambda}^{(2)}-1\rangle,
\end{equation*}
where the $\log$ is taken elementwise. This function is composed of two terms: the first term keeps $\mathbold{\lambda}^{(1)}$ unchanged while the second term handles the non-negative constraint of $\mathbold{\lambda}^{(2)}$, which corresponds to the variance (see \cite{beck2003mirror} or Appendix \ref{appendix:mdalgorithm} for the background of the mirror descent). Then the mirror map induced by this convex function $\varphi$ is defined as follows:
\begin{equation}
\label{eqn:mirrormap}
    \nabla \varphi(\mathbold{\lambda})=\mathbold{\zeta},
\end{equation}
where $\mathbold{\zeta}\in\mathbb{R}^{d}$ is defined as: $\mathbold{\zeta}^{(1)}=\mathbold{\lambda}^{(1)}$ and $\mathbold{\zeta}^{(2)}=\log \mathbold{\lambda}^{(2)}$. The inverse of the mirror map is defined as follows:
\begin{equation}
\label{eqn:inversemap}
    \nabla \varphi^{*}(\mathbold{\zeta})=\mathbold{\lambda},
\end{equation}
where $\mathbold{\lambda}^{(1)}=\mathbold{\zeta}^{(1)}$, $\mathbold{\lambda}^{(2)}=\exp(\mathbold{\zeta}^{(2)})$ (with $\exp$ taken elementwise). The dual function of $\varphi$ is denoted as $\varphi^{*}$. The basic idea of our solution is to map the parameters $\mathbold{\lambda}_{k,n}$ (for $k=1,2,...,K$) to the dual space using the mirror map defined by \eqref{eqn:mirrormap} before each update. After updating these parameters in the given direction, we map them back to the original space using the inverse map defined by \eqref{eqn:inversemap}. This ensures that the updated parameters always belong to the constrained domain. In summary, we modify the updates 
as follows: for GFlowVI, we have
\begin{align}
\label{eqn:modifedngflow-vi}
\begin{split}
    &\textbf{s}^\prime_{k,n} = \log (\textbf{s}_{k,n}).\\
    &\textbf{s}^\prime_{k,n+1} = \textbf{s}^\prime_{k,n} -\eta \textbf{w}_{k}\nabla_{\textbf{s}_{k}}\log q(\textbf{z}|\mathbold{\lambda}_{k,n}).\\
    &+\frac{\eta}{2} \oslash (\textbf{s}_{k,n}\odot\textbf{s}_{k,n})\odot\texttt{diag}\left[\nabla^{2}_{\textbf{z}}f(\textbf{z})+\nabla^{2}_{\textbf{z}}\log q_{n}(\textbf{z})\right]\\
    &\textbf{s}_{k,n+1} = \exp(\textbf{s}^\prime_{k,n+1}).\\
    &\mathbold{\mu}_{k,n+1}= \mathbold{\mu}_{k,n}-\eta \left[\nabla_{\textbf{z}}f(\textbf{z})+\nabla_{\textbf{z}}\log q_{n}(\textbf{z})\right]\\
    &- \eta\textbf{w}_{k}\nabla_{\mathbold{\mu}_{k}}\log q(\textbf{z}|\mathbold{\lambda}_{k,n}),\\
\end{split}
\end{align}

where the first line is to map the vectors $\textbf{s}_{k,n}$ to the dual space through the mirror map \eqref{eqn:mirrormap}, the second line is to update these vectors in the dual space, and the third line is to map the updated variance vectors back to the constrained domain through the inverse map \eqref{eqn:inversemap}. It can be confirmed that the variance vectors are always positive.

For NGFlowVI, we apply the same procedure to update the variance vectors. The modification of the update (\ref{eqn:ngflow-vi}) can be expressed as follows:
\begin{align}
\label{eqn:modifedngflow-vi}
\begin{split}
    &\textbf{s}^\prime_{k,n} = \log (\textbf{s}_{k,n}).\\
    &\textbf{s}^\prime_{k,n+1} = \textbf{s}^\prime_{k,n} + \eta \texttt{diag}\left[ \nabla^{2}_{\textbf{z}}f(\textbf{z})+\nabla^{2}_{\textbf{z}}\log q_{n}(\textbf{z})\right]\\
    & -  2\eta \textbf{w}_{k}(\textbf{s}_{k,n}\odot \textbf{s}_{k,n})\odot\nabla_{\textbf{S}_{k}}\log q(\textbf{z}|\lambda_{k,n}).\\
    &\textbf{s}_{k,n+1} = \exp(\textbf{s}^\prime_{k,n+1}).\\
    &\mathbold{\mu}_{k,n+1}= \mathbold{\mu}_{k,n}-\eta \left[\nabla_{\textbf{z}}f(\textbf{z})+\nabla_{\textbf{z}}\log q_{n}(\textbf{z})\right]\oslash\textbf{s}_{k,n+1}\\
    &- \eta \textbf{w}_{k}\nabla_{\mathbold{\mu}_{k}}\log q(\textbf{z}|\mathbold{\lambda}_{k,n})\oslash\textbf{s}_{k,n+1}.\\
\end{split}
\end{align}
\noindent
\textbf{Remark 4}. The approach outlined above can be generalized to handle the constraints imposed on $\mathbold{\lambda}^{(2)}$ such as $a\leq\mathbold{\lambda}^{(2)}\leq b$. We can modify the convex function $\varphi$ as follows:
\begin{align*}
\begin{split}
    \varphi(\mathbold{\lambda})=\frac{1}{2}\lVert \mathbold{\lambda}^{(1)} \rVert_{2}^{2}+& \langle \mathbold{\lambda}^{(2)}-a,\log( \mathbold{\lambda}^{(2)}-a)\rangle\\
    +& \langle b-\mathbold{\lambda}^{(2)},\log(b- \mathbold{\lambda}^{(2)})\rangle.
\end{split}
\end{align*}
The mirror map in (\ref{eqn:mirrormap}) for $\mathbold{\lambda}^{(2)}$ is modified as follows:  $\mathbold{\zeta}^{(2)}=\log( \mathbold{\lambda}^{(2)}-a)-\log( b-\mathbold{\lambda}^{(2)})$. Also the inverse of the mirror map (\ref{eqn:inversemap}) for $\mathbold{\lambda}^{(2)}$  is modified as: $\mathbold{\lambda}^{(2)}=(b\exp(\mathbold{\zeta}^{(2)})+a)/(\exp(\mathbold{\zeta}^{(2)})+1)$, thus, $\mathbold{\lambda}^{(2)}$ always satisfies the constraints during the optimization process, i.e., $a\leq\mathbold{\lambda}^{(2)}\leq b$.

\begin{figure*}
    \centering
    \includegraphics[width=1.0\textwidth]{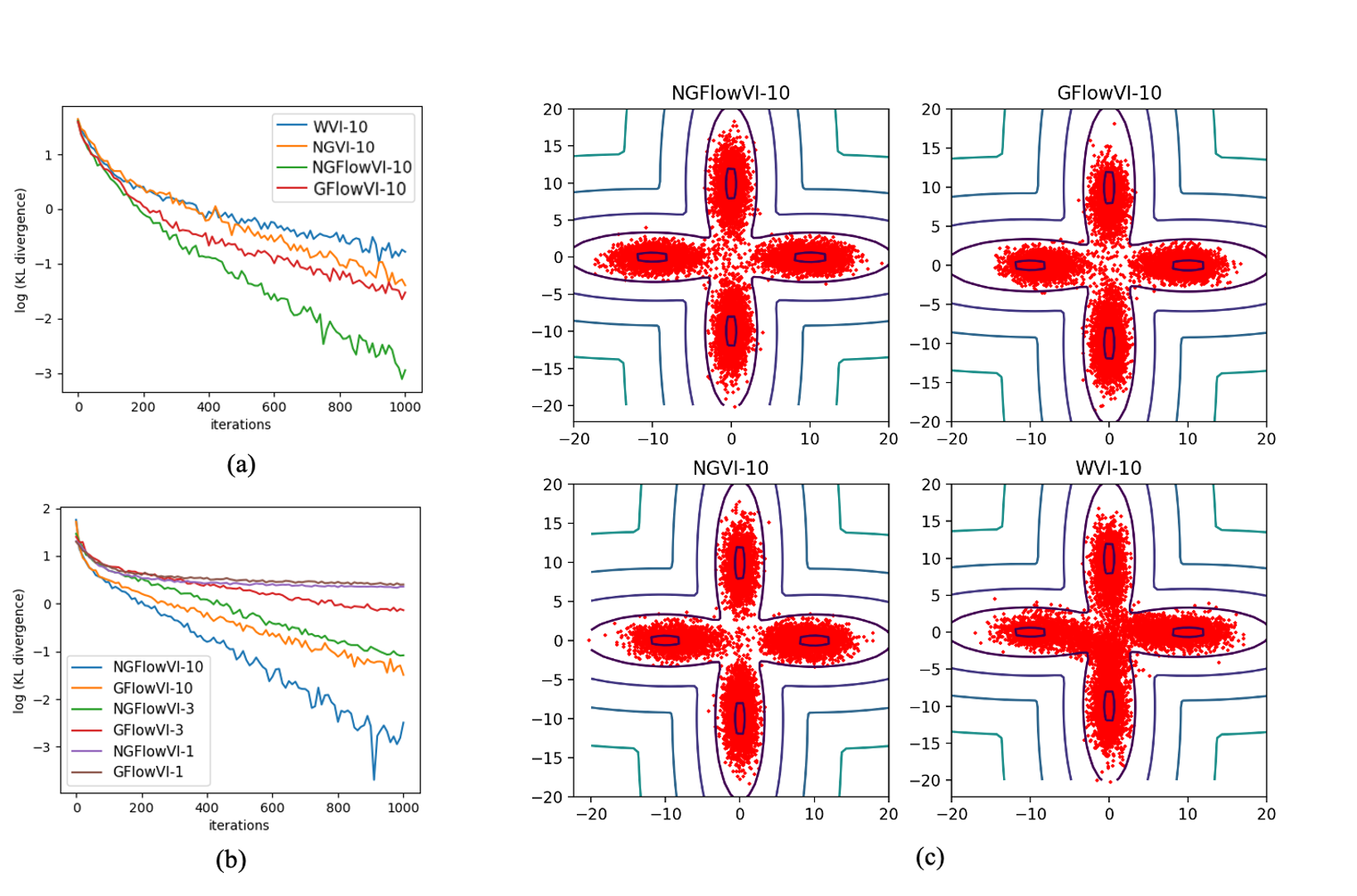}
    \caption{Experimental results on the synthetic dataset: (a) the estimated KL divergence in log scale between the target $\pi$ and approximate density $q$ over 1,000 iterations of four updates with $K=10$; (b) performance of NGFlowVI and GFlowVI with varying values of $K$: 1, 3 and 5; (c) visualizations of 1,000 samples from $q$ given by the four updates.}
    \label{fig:simulation-10}
\end{figure*}

\section{Numerical Experiments}\label{sec:exps}
To validate and enhance our theoretical analyses of the proposed updates, we conduct a series of numerical experiments, including simulated experiments and applications to Bayesian neural networks. We compare GFlowVI and NGFlowVI with  Wasserstein variational inference (WVI) \citep{lambert2022variational} and natural gradient variational inference for mixture models (NGVI) \citep{lin2019fast}. We omit BBVI from our experiments, as previous work demonstrated that NGVI is superior to BBVI. 
We use numbers following the methods' names to denote the number of components $K$ in the mixtures for each method. Experiments are done on a PC with Intel Core i9 and 64 GB memory.

\begin{figure*}
\centering
\begin{subfigure}{0.32\textwidth}
    \includegraphics[width=\textwidth]{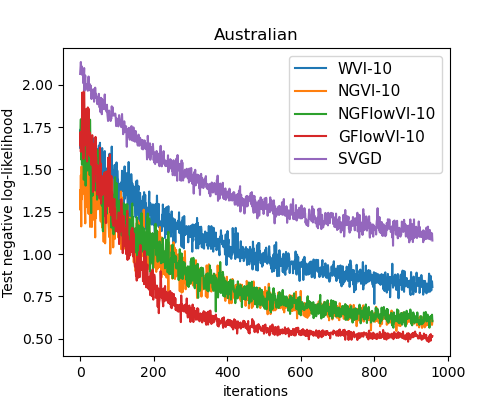}
    \caption{}
    \label{fig:first}
\end{subfigure}
\begin{subfigure}{0.32\textwidth}
    \includegraphics[width=\textwidth]{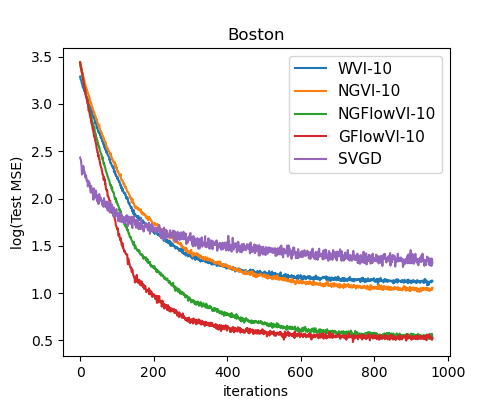}
    \caption{}
    \label{fig:second}
\end{subfigure}
\begin{subfigure}{0.32\textwidth}
    \includegraphics[width=\textwidth]{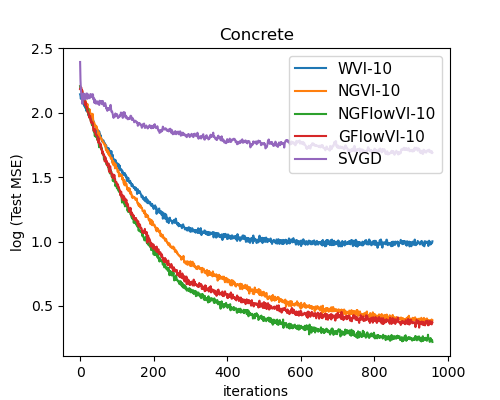}
    \caption{}
    \label{fig:third}
\end{subfigure}
\caption{Average test negative log-likelihood of Bayesian neural networks (BNNs) on (a) 'Australia scale' and averaged test mean square error of BNNs on (b) 'Boston' and (c) 'Concrete' over 1000 iterations. For SVGD, 100 particles are used, while other methods approximate BNN weight posteriors with a Gaussian mixture ($K=10$). Parameters are updated using WVI-10, NGVI-10, GFlowVI-10 and NGFlowVI-10. Results are averaged over 20 runs of 20 data splits.}
\label{fig:testerror}
\end{figure*}

\noindent
\textbf{Results on simulated datasets}. We first consider to sample from a 4-cluster Gaussian mixture distribution $\pi$, defined on a two-dimensional space, with equal cluster weights. The objective is to minimize the KL divergence between $q$ and $\pi$. For NGVI, GFlowVI and NGFlowVI, $q$ is a mixture of diagonal Gaussians, while for WVI, $q$ is a mixture of Gaussians with full covariance. Initially, the means of Gaussians are randomly sampled from a two-dimensional normal distribution and variances are set to the identity matrix. Means and covariance matrices are updated over 1,000 iterations with a fixed learning rate $\eta= 0.001$. The expectation $\mathbb{E}_{\textbf{z}\sim q(\cdot|\mathbold{\lambda})}[\cdot]$ is estimated using a single sample.

Figure \ref{fig:simulation-10}(a) shows the KL divergence between $q$ and $\pi$ over 1,000 iterations for the four updates when $K=10$. The KL divergence is estimated by $10^{4}$ MC samples from $q$. We see that NGFlowVI-10 converges faster than the others. GFlowVI-10's convergence is comparable to NGVI-10, and much faster than WVI-10. We also evaluate NGFlowVI and GFlowVI for $K=1,3,10$, as shown in Figure \ref{fig:simulation-10}(b). With $K=1$, both NGFlowVI and GFlowVI perform poorly, as a single Gaussian might not capture the multi-modal target distribution. However, with $K=10$, both methods effectively approximate the target. Figure \ref{fig:simulation-10}(c) shows 1,000 samples from the approximate density $q$ for WVI-10, NGVI-10, GFlowVI-10 and NGFlowVI-10. These evidences confirm the effectiveness of NGFlowVI and GFlowVI on this synthetic dataset. In addition, we apply these methods to approximate two other synthetic distributions defined on a two-dimensional space: a banana-shaped distribution and an X-shaped mixture of Gaussians. The densities of these distributions, the approximate KL divergence between $q$ and the targets after 1000 particle updates, and the visualizations of 1000 samples are given in Table \ref{tab:targets}, Table \ref{tab:synthetic} and Figure \ref{fig:synthetic}, respectively, in Appendix \ref{appendix:syntheticdatasets}. 

\noindent
\textbf{Results on real-world datasets}. We also validate our methods on Bayesian neural networks using real-world datasets and include SVGD \citep{liu2016stein}) for comparison. 
We use the following three datasets:  1) 'Australian': $N =790$ examples , dimensionality$=14$, with 345 for training; 2) 'Boston': $N=506$, dimensionality$=8$, with 455 for training; 3) 'Concrete': $N=1030$, dimensionality$=13$, with 927 for training. We perform classification on the first dataset and regression on the others, using 20 data splits provided by \cite{gal2016dropout}. Results are averaged over 20 runs of these splits.
We employ the same deep neural network architecture for all datasets with one hidden layer, 50 hidden units and ReLU activation. The regularization parameter and learning rate are set to 0.1 and 0.0001, respectively. We use minibatches of size 32 to approximate gradients and Hessians of $f$ in \eqref{eqn:negatedelbo}. The posterior over the network weights is approximated by a Gaussian mixture with $K=10$, and parameters are updated through 1,000 iterations. For predictions, we draw 100 samples of weights for the networks and calculate the average prediction for the given input. We use 100 particles for SVGD. We use 10 samples to estimate the expectation $\mathbb{E}_{\textbf{z}\sim q(\cdot|\mathbold{\lambda})}[\cdot]$. Figure \ref{fig:testerror} shows the averaged negative log-likelihood over 1,000 iterations. GFlowVI-10 achieves the best convergence on the first two datasets, while NGFlowVI-10 and NGVI-10 leads on the third. See more details on the experiments in Appendix \ref{appendix:results}. 

\section{Conclusions}
\label{sec:conclusion}
We introduced a novel WGF-based approach for VI that operates on variational parameter domains, unlike previous methods, which focus on latent variable domains. This approach makes significant contributions to related fields. Our developed algorithms were empirically validated on the synthetic and real-world datasets, demonstrating their effectiveness. However, our current work is limited to diagonal Gaussian distributions for the algorithmic development. There are two main reasons. First, a single Gaussian may not adequately capture the complexity of the posterior (e.g. rotated Gaussian), but a mixture of diagonal Gaussian can provide a better approximation. Second, the proposed updates (\ref{eqn:mirrorWGD}) might violate the parameter constraints, such as ensuring that the covariance matrix must be positive definite. To address this issue, we opted for the simpler case of diagonal covariance and employed the mirror descent for constrained optimization. We plan to address  full covariance Gaussians in future work.




\subsubsection*{Acknowledgements}
This work was supported in part by MEXT KAKENHI [grant number: 23K16939] (to D.H.N.) and MEXT KAKENHI [grant numbers: 22H03645 and 23K24901] (to H.M.).

\bibliographystyle{plainnat}
\bibliography{aistats}

\section*{Checklist}

 \begin{enumerate}

 \item For all models and algorithms presented, check if you include:
 \begin{enumerate}
   \item A clear description of the mathematical setting, assumptions, algorithm, and/or model. [Yes]
   \item An analysis of the properties and complexity (time, space, sample size) of any algorithm. [Yes]
   \item (Optional) Anonymized source code, with specification of all dependencies, including external libraries. [Not Applicable]
 \end{enumerate}

 \item For any theoretical claim, check if you include:
 \begin{enumerate}
   \item Statements of the full set of assumptions of all theoretical results. [Yes]
   \item Complete proofs of all theoretical results. [Yes]
   \item Clear explanations of any assumptions. [Yes]     
 \end{enumerate}

 \item For all figures and tables that present empirical results, check if you include:
 \begin{enumerate}
   \item The code, data, and instructions needed to reproduce the main experimental results (either in the supplemental material or as a URL). [Yes]
   \item All the training details (e.g., data splits, hyperparameters, how they were chosen). [Yes]
         \item A clear definition of the specific measure or statistics and error bars (e.g., with respect to the random seed after running experiments multiple times). [Yes]
         \item A description of the computing infrastructure used. (e.g., type of GPUs, internal cluster, or cloud provider). [Yes]
 \end{enumerate}

 \item If you are using existing assets (e.g., code, data, models) or curating/releasing new assets, check if you include:
 \begin{enumerate}
   \item Citations of the creator If your work uses existing assets. [[Not Applicable]]
   \item The license information of the assets, if applicable. [Not Applicable]
   \item New assets either in the supplemental material or as a URL, if applicable. [Not Applicable]
   \item Information about consent from data providers/curators. [Not Applicable]
   \item Discussion of sensible content if applicable, e.g., personally identifiable information or offensive content. [Not Applicable]
 \end{enumerate}

 \item If you used crowdsourcing or conducted research with human subjects, check if you include:
 \begin{enumerate}
   \item The full text of instructions given to participants and screenshots. [Not Applicable]
   \item Descriptions of potential participant risks, with links to Institutional Review Board (IRB) approvals if applicable. [Not Applicable]
   \item The estimated hourly wage paid to participants and the total amount spent on participant compensation. [Not Applicable]
 \end{enumerate}

 \end{enumerate}

\onecolumn
\appendix
\date{}
\title{\textbf{Appendix for Wasserstein Gradient Flow over Variational Parameter Space for Variational
Inference}}
\maketitle

\section{Proof of Theorem \ref{thm:firstvariation}}
\label{theorem1}
\begin{proof}
    To compute the first variation of $\mathcal{L}(p)$, suppose that $\varepsilon>0$ and an arbitrary distribution $\chi\in \mathcal{P}(\Omega)$. We compute $(\mathcal{L}(\rho+\varepsilon \chi)-\mathcal{L}(\rho))/\varepsilon$ as follows:
    \begin{align*}
        &\frac{1}{\varepsilon}\left[\mathcal{L}(\rho+\varepsilon \chi) - \mathcal{L}(\rho)\right]=\\
        &\frac{1}{\varepsilon}\int\left( \rho(\mathbold{\lambda})+\varepsilon \chi(\mathbold{\lambda})\right)\int q(\textbf{z}|\mathbold{\lambda})\left[f(\textbf{z})+\log\left( \int\left( \rho(\mathbold{\lambda})+\varepsilon \chi(\mathbold{\lambda})\right)q(\textbf{z}|\mathbold{\lambda})d\mathbold{\mathbold{\lambda}}\right)\right]d\textbf{z}d\mathbold{\lambda}\\
        -& \frac{1}{\varepsilon}\int\rho(\mathbold{\lambda})\int q(\textbf{z}|\mathbold{\lambda})\left[f(\textbf{z})+\log\left( \int\rho(\mathbold{\lambda})q(\textbf{z}|\mathbold{\lambda})d\mathbold{\lambda}\right)\right]d\textbf{z}d\mathbold{\lambda}\\
        =& \int\chi(\mathbold{\lambda})\int q(\textbf{z}|\mathbold{\lambda})f(\textbf{z})d\textbf{z}d\mathbold{\lambda}
        + \frac{1}{\varepsilon}\int\left( \rho(\mathbold{\lambda})+\varepsilon \chi(\mathbold{\lambda})\right)\int q(z|\mathbold{\lambda})\log\left( \int\left( \rho(\mathbold{\lambda})+\varepsilon \chi(\mathbold{\lambda})\right)q(\textbf{z}|\mathbold{\lambda})d\mathbold{\mathbold{\lambda}}\right)d\textbf{z}d\mathbold{\lambda}\\
        -&\frac{1}{\varepsilon}\int\rho(\mathbold{\lambda})\int q(\textbf{z}|\mathbold{\lambda})\log\left( \int\rho(\mathbold{\lambda})q(\textbf{z}|\mathbold{\lambda})d\mathbold{\lambda}\right)d\textbf{z}d\mathbold{\lambda}\\
        =& \int\chi(\mathbold{\lambda})\int q(\textbf{z}|\mathbold{\lambda})f(\textbf{z})d\textbf{z}d\mathbold{\lambda}\\
        +& \frac{1}{\varepsilon}\int\left( \rho(\mathbold{\lambda})+\varepsilon \chi(\mathbold{\lambda})\right)\int q(\textbf{z}|\mathbold{\lambda})\log\left( \int\left( \rho(\mathbold{\lambda})+\varepsilon \chi(\mathbold{\lambda})\right)q(z|\mathbold{\lambda})d\mathbold{\mathbold{\lambda}}\right)d\textbf{z}d\mathbold{\lambda} - \\
        &\frac{1}{\varepsilon}\int\rho(\mathbold{\lambda})\int q(\textbf{z}|\mathbold{\lambda})\log \left(\int\left( \rho(\mathbold{\lambda})+\varepsilon \chi(\mathbold{\lambda})\right)q(\textbf{z}|\mathbold{\lambda})d\mathbold{\lambda} \right)d\textbf{z}d\mathbold{\lambda}\\ 
        +& \frac{1}{\varepsilon}\left[ \int\rho(\mathbold{\lambda})\int q(\textbf{z}|\mathbold{\lambda})\log \left(\int\left( \rho(\mathbold{\lambda})+\varepsilon \chi(\mathbold{\lambda})\right)q(\textbf{z}|\mathbold{\lambda})d\mathbold{\lambda} \right)d\textbf{z}d\mathbold{\lambda}-\int\rho(\mathbold{\lambda})\int q(\textbf{z}|\mathbold{\lambda})\log\left( \int\rho(\mathbold{\lambda})q(z|\mathbold{\lambda})d\mathbold{\lambda}\right)d\textbf{z}d\mathbold{\lambda}\right]\\
        =& \int\chi(\mathbold{\lambda})\int q(\textbf{z}|\mathbold{\lambda})f(\textbf{z})d\textbf{z}d\mathbold{\lambda}\\
        +& \underbrace{\int\chi(\mathbold{\lambda})\int q(\textbf{z}|\mathbold{\lambda})\log\left(\int \left(\rho(\mathbold{\lambda}) + \varepsilon\chi(\mathbold{\lambda})\right)q(\textbf{z}|\mathbold{\lambda})d\mathbold{\lambda}\right)d\textbf{z}d\mathbold{\lambda}}_{\text{(a)}}\\
        +& \underbrace{\frac{1}{\varepsilon}\left[ \int\rho(\mathbold{\lambda})\int q(\textbf{z}|\mathbold{\lambda})
        \log \left(1 + \frac{\varepsilon \int\chi(\mathbold{\lambda})q(\textbf{z}|\mathbold{\lambda})d\mathbold{\lambda}}{\int \rho(\mathbold{\lambda})q(\textbf{z}|\mathbold{\lambda})d\mathbold{\lambda}} \right)d\textbf{z}d\mathbold{\lambda}
        \right]}_{\text{(b)}}.
    \end{align*}
    We process parts (a) and (b), when $\varepsilon\rightarrow 0$, as follows:
    \begin{align*}
        \lim_{\varepsilon\rightarrow 0}\text{(a)} =& \int\chi(\mathbold{\lambda})\int q(\textbf{z}|\mathbold{\lambda})\log\left(\int \rho(\mathbold{\lambda})q(\textbf{z}|\mathbold{\lambda})d\mathbold{\lambda}\right)d\textbf{z}d\mathbold{\lambda},\\
        \lim_{\varepsilon\rightarrow 0}\text{(b)} =& \int \rho(\mathbold{\lambda})\int q(\textbf{z}|\mathbold{\lambda})
       \frac{\int\chi(\mathbold{\lambda})q(\textbf{z}|\mathbold{\lambda})d\mathbold{\lambda}}{\int\rho(\mathbold{\lambda})q(\textbf{z}|\mathbold{\lambda})d\mathbold{\lambda}}d\textbf{z}d\mathbold{\lambda}       =\int\int \rho(\mathbold{\lambda})q(\textbf{z}|\mathbold{\lambda})d\mathbold{\lambda}\frac{\int\chi(\mathbold{\lambda})q(\textbf{z}|\mathbold{\lambda})d\mathbold{\lambda}}{\int \rho(\mathbold{\lambda})q(\textbf{z}|\mathbold{\lambda})d\mathbold{\lambda}}d\textbf{z}\\
       =& \int\chi(\mathbold{\lambda})\int  q(\textbf{z}|\mathbold{\lambda})d\textbf{z}d\mathbold{\lambda}= \int\chi(\mathbold{\lambda})d\mathbold{\lambda},
    \end{align*}
    where we have used the following equality for (b): $\lim_{\varepsilon\rightarrow0}\frac{\log \left(1 + \varepsilon x \right)}{\varepsilon}=x$ for all $x\in \mathbb{R}$.
    
    So, we have:
    \begin{align*}
        \lim_{\varepsilon\rightarrow 0}\frac{1}{\varepsilon}\left[\mathcal{L}(\rho+\varepsilon \chi) - \mathcal{L}(\rho)\right]=\int\chi(\mathbold{\lambda})\left(\mathbb{E}_{\textbf{z}\sim q(\cdot|\mathbold{\lambda})}\left[ f(\textbf{z}) + \log q(\textbf{z}) \right] + 1\right)d\mathbold{\lambda}.
    \end{align*}
    By definition of the first variation of $\mathcal{L}$, this completes the proof of Theorem \ref{thm:firstvariation}.
\end{proof}

\section{Proof of Proposition \ref{lemma:continuoustimedynamics}}
\label{proof:continuoustimedynamics}
\begin{proof}
    Using the differential calculus in the Wasserstein space and the chain rule, we have:
\begin{align*}
    \frac{d \mathcal{L}(\rho_{t})}{d t}=& -\int \delta \mathcal{L}(\rho_{t})(\mathbold{\lambda})\texttt{div}(\rho_{t}\textbf{C}v_{t})d\mathbold{\lambda} 
    = \int\langle \textbf{C}(\mathbold{\lambda})v_{t}(\mathbold{\lambda}), \nabla_{\mathbold{\lambda}} \delta \mathcal{L}(\rho_{t}) (\mathbold{\lambda})\rangle d \rho_{t}(\mathbold{\lambda})\\
    =& -\int\langle v_{t}(\mathbold{\lambda}),\textbf{C}(\mathbold{\lambda})v_{t}(\mathbold{\lambda}) \rangle d \rho_{t}(\mathbold{\lambda}).
\end{align*}
\end{proof}
\section{Proof of Proposition \ref{lemma:descentlemma}}
\label{appendix:descentlemma}
\begin{proof}
    Denote $v_{n}(\mathbold{\lambda})=-\nabla_{\mathbold{\lambda}}\mathbb{E}_{\textbf{z}\sim q(\cdot|\mathbold{\lambda})}\left[ f(\textbf{z}) + \log q_{n}(\textbf{z}) \right]$ where $q_{n}(\textbf{z})=\mathbb{E}_{\mathbold{\lambda}\sim \rho_{n}}\left[q(\textbf{z}|\mathbold{\lambda})\right]$, $\Phi_{t}(\mathbold{\lambda})=\mathbold{\lambda} + t \textbf{C}(\mathbold{\lambda})v_{n}(\mathbold{\lambda})$ for $t\in [0,\eta]$, and $\nu_{t}=(\Phi_{t})_{\#}\rho_{n}$. Then it is evident that $\nu_{0}=\rho_{n}$ and $\nu_{\eta}=\rho_{n+1}$.

    We define $\phi(t)=\mathcal{L}(\nu_{t})$. Clearly, $\phi(0)=\mathcal{L}(\rho_{n})$ and $\phi(\eta)=\mathcal{L}(\rho_{n+1})$. Using a Taylor expansion, we have:
    \begin{equation}
    \label{eqn:taylor}
        \phi(\eta)=\phi(0) + \eta \phi^\prime(0) + \int_{0}^{\eta}(\eta - t)\phi^{\prime\prime}(t)dt.
    \end{equation}
    Using the chain rule, we can estimate $\phi^\prime(t)$ as follows:
    \begin{align*}
        \phi^\prime(t)=& \frac{d}{dt}\int\rho_{t}(\mathbold{\lambda})\int q(\textbf{z}|\mathbold{\lambda})\left[ f(\textbf{z})+\log \left( \int_{\mathbold{\lambda}}\rho_{t}(\mathbold{\lambda})q(\textbf{z}|\mathbold{\lambda})d\mathbold{\lambda}\right)\right]d\textbf{z}d\mathbold{\lambda}\\
        =& \frac{d}{dt}\int\rho_{n}(\mathbold{\lambda})\int q(\textbf{z}|\Phi_{t}(\mathbold{\lambda}))\left[ f(\textbf{z})+\log \left( \int\rho_{n}(\mathbold{\lambda})q(\textbf{z}|\Phi_{t}(\mathbold{\lambda}))d\mathbold{\lambda}\right)\right]d\textbf{z}d\mathbold{\lambda}\\
        =& \int\rho_{n}(\mathbold{\lambda})\langle \frac{d\Phi_{t}(\mathbold{\lambda})}{dt}, \nabla_{\mathbold{\lambda}}\int q(\textbf{z}|\Phi_{t}(\mathbold{\lambda}))\left[f(\textbf{z})+\log \left(\int\rho_{n}(\mathbold{\lambda})q(\textbf{z}|\Phi_{t}(\mathbold{\lambda}))d\mathbold{\lambda} \right) \right]d\textbf{z} \rangle d\mathbold{\lambda}\\
        =& \int\rho_{n}(\mathbold{\lambda})\langle \textbf{C}(\mathbold{\lambda})v_{n}(\mathbold{\lambda}), \nabla_{\mathbold{\lambda}}\mathbb{E}_{\textbf{z}\sim q(\cdot|\Phi_{t}(\mathbold{\lambda}))}\left[ f(\textbf{z}) + \log q_{t}(\textbf{z}) \right] \rangle d\mathbold{\lambda},
    \end{align*}
    where $q_{t}(\textbf{z})=\int\rho_{n}(\mathbold{\lambda})q(\textbf{z}|\Phi_{t}(\mathbold{\lambda}))d\mathbold{\lambda}$. The second equality is obtained by applying the change of variable formula, and the last equality is obtained by the definition of $v_{n}$ and $\Phi_{t}$.

    So, at $t=0$, we have:
    \begin{equation}
     \label{eqn:grad}
        \phi^\prime(0)= -\int\langle  v_{n}(\mathbold{\lambda}), \textbf{C}(\mathbold{\lambda})v_{n}(\mathbold{\lambda})\rangle \rho_{n}(\mathbold{\lambda})d\mathbold{\lambda}=-\langle v_{n}, \textbf{C}v_{n} \rangle_{L^{2}(p_{n})}.
    \end{equation}

    Next we estimate $\phi^{\prime\prime}(\mathbold{\lambda})$ as follows:
    \begin{align*}
        \phi^{\prime\prime}(\mathbold{\lambda})=& \frac{d}{dt} \phi^{\prime}(t)=\int \rho_{n}(\mathbold{\lambda})\langle \textbf{C}(\mathbold{\lambda})v_{n}(\mathbold{\lambda}), \frac{d}{dt}\nabla_{\mathbold{\lambda}}\mathbb{E}_{\textbf{z}\sim q(\cdot|\Phi_{t}(\mathbold{\lambda}))}\left[ f(\textbf{z}) + \log q_{t}(\textbf{z}) \right]\rangle d\mathbold{\lambda}\\
        =& \int \rho_{n}(\mathbold{\lambda})\langle \textbf{C}(\mathbold{\lambda})v_{n}(\mathbold{\lambda}), \textbf{H}_{t}(\mathbold{\lambda})v_{n}(\mathbold{\lambda})\rangle d\mathbold{\lambda} =\langle \textbf{C}v_{n}, \textbf{H}_{t}v_{n} \rangle_{L^{2}(p_{n})}, 
    \end{align*}
    where $\textbf{H}_{t}(\mathbold{\lambda})=\nabla^{2}_{\mathbold{\lambda}}\mathbb{E}_{\textbf{z}\sim q(\cdot|\Phi_{t}(\mathbold{\lambda}))}\left[  f(\textbf{z}) + \log q_{t}(\textbf{z})\right]$. Now we need to upper-bound the operator norm of $\textbf{H}_{t}$. Denote $\mathbold{\lambda}_{t}=\Phi_{t}(\mathbold{\lambda})$, we can rewrite 
    $\textbf{H}_{t}$ as follows:
    \begin{align*}
        \textbf{H}_{t}(\mathbold{\lambda}) =& \nabla_{\mathbold{\lambda}}\left[\int\nabla_{\mathbold{\lambda}}q(\textbf{z}|\mathbold{\lambda}_{t})\left[ f(\textbf{z})+\log q_{t}(\textbf{z})\right]d\textbf{z} + \int  q(\textbf{z}|\mathbold{\lambda}_{t})\nabla_{\mathbold{\lambda}}\log q_{t}(\textbf{z})d\textbf{z}\right]\\
        =& \int\nabla_{\mathbold{\lambda}}^{2}q(\textbf{z}|\mathbold{\lambda}_{t})\left[ f(\textbf{z})+\log q_{t}(\textbf{z})\right]d\textbf{z},
    \end{align*}
    where the second equality is obtained using the fact: $\nabla_{\mathbold{\lambda}}\log q_{t}(\textbf{z})=0$. Therefore:
    \begin{align*}
        &\textbf{H}_{t}(\mathbold{\lambda}) = \int\nabla_{\mathbold{\lambda}}\left[ q(\textbf{z}|\mathbold{\lambda}_{t})\nabla_{\mathbold{\lambda}}\log q(\textbf{z}|\mathbold{\lambda}_{t}) \right]\left[ f(\textbf{z}) + \log q_{t}(\textbf{z}) \right]d\textbf{z}\\
        =& \int\left[ \nabla_{\mathbold{\lambda}} q(\textbf{z}|\mathbold{\lambda}_{t})\right] \left[ \nabla_{\mathbold{\lambda}}\log q(\textbf{z}|\mathbold{\lambda}_{t})\right]^{\intercal}\left[ f(\textbf{z}) + \log q_{t}(\textbf{z}) \right]d\textbf{z}
        + \int q(\textbf{z}|\mathbold{\lambda}_{t})\nabla_{\lambda}^{2}\log q(\textbf{z}|\mathbold{\lambda}_{t})\left[f(\textbf{z}) + \log q_{t}(\textbf{z}) \right]d\textbf{z}\\
        =& \int q(\textbf{z}|\mathbold{\lambda}_{t})\left[ \nabla_{\mathbold{\lambda}} \log q(\textbf{z}|\mathbold{\lambda}_{t})\right] \left[ \nabla_{\mathbold{\lambda}}\log q(\textbf{z}|\mathbold{\lambda}_{t})\right]^{\intercal}\left[ f(\textbf{z}) + \log q_{t}(\textbf{z}) \right]d\textbf{z}\\
        +& \int q(\textbf{z}|\mathbold{\lambda}_{t})\nabla_{\lambda}^{2}\log q(\textbf{z}|\mathbold{\lambda}_{t})\left[f(\textbf{z}) + \log q_{t}(\textbf{z}) \right]d\textbf{z}.
    \end{align*}
    
    Then, the operator norm of $\textbf{H}_{t}(\mathbold{\lambda})$ can be upper-bounded as follows:
    \begin{align}
    \label{eqn:parta}
    \begin{split}
    &\lVert \textbf{H}_{t}(\mathbold{\lambda}) \rVert_{\text{op}} \leq  \mathbb{E}_{\textbf{z}\sim q(\cdot|\mathbold{\lambda}_{t})}\lVert \nabla_{\mathbold{\lambda}}\log q(\textbf{z}|\mathbold{\lambda}_{t})\rVert_{2}^{2}|f(\textbf{z}) + \log q_{t}(\textbf{z})| + \mathbb{E}_{\textbf{z}\sim q(\cdot|\mathbold{\lambda}_{t})}\lVert \nabla_{\mathbold{\lambda}}^{2}\log q(\textbf{z}|\mathbold{\lambda}_{t}) \rVert_{\text{op}}|f(\textbf{z})+\log q_{t}(\textbf{z})|\\
    \leq & \left[\mathbb{E}_{\textbf{z}\sim q(\cdot|\mathbold{\lambda}_{t})}\lVert \nabla_{\mathbold{\lambda}}\log q(\textbf{z}|\mathbold{\lambda}_{t})\rVert_{2}^{2}\right] \left[\mathbb{E}_{\textbf{z}\sim q(\cdot|\mathbold{\lambda}_{t})}|f(\textbf{z})| + \mathbb{E}_{\textbf{z}\sim q(\cdot|\mathbold{\lambda}_{t})}|\log q_{t}(\textbf{z})| \right]\\
    +&  \left[ \mathbb{E}_{\textbf{z}\sim q(\cdot|\mathbold{\lambda}_{t})}\lVert \nabla_{\mathbold{\lambda}}^{2}\log q(\textbf{z}|\mathbold{\lambda}_{t}) \rVert_{\text{op}}\right] \left[ \mathbb{E}_{\textbf{z}\sim q(\cdot|\mathbold{\lambda}_{t})} |f(\textbf{z})|+\mathbb{E}_{\textbf{z}\sim q(\cdot|\mathbold{\lambda}_{t})}|\log q_{t}(\textbf{z})|\right]\\
    \leq & (\alpha+\beta)(M_{1}+M_{2}),
    \end{split}
    \end{align}
    where the first inequality is obtained using the equality $\lVert \textbf{a}\textbf{b}^{\intercal}\rVert_{\text{op}}= \lVert \textbf{a} \rVert_{2}\lVert \textbf{b} \rVert_{2}$ for two vectors $\textbf{a}$ and $\textbf{b}$; the second inequality is obtained by using the inequality $\mathbb{E}_{\textbf{z}\sim q(\cdot|\mathbold{\lambda})}|h_{1}(\textbf{z})||h_{2}(\textbf{z})|\leq \mathbb{E}_{\textbf{z}\sim q(\cdot|\mathbold{\lambda})}|h_{1}(\textbf{z})|\mathbb{E}_{\textbf{z}\sim q(\cdot|\mathbold{\lambda})}|h_{2}(\textbf{z})|$ for two scalar functions $h_{1}$ and $h_{2}$; the last inequality is obtained by using assumptions \textbf{(A1)}, \textbf{(A2)} and \textbf{(A3)}.

    Thus, plugging the results \eqref{eqn:grad} and \eqref{eqn:parta} into \eqref{eqn:taylor}, we can derive the following inequality:
    \begin{align*}
        \mathcal{L}(\rho_{n+1}) \leq & \mathcal{L}(\rho_{n}) -\eta \langle v_{n},\textbf{C}v_{n}\rangle_{L^{2}(\rho_{n})} 
        + \int_{0}^{\eta}(\eta - t)\kappa \langle v_{n},\textbf{C}v_{n}\rangle_{L^{2}(\rho_{n})}dt\\
        =& \mathcal{L}(\rho_{n}) -\eta \langle v_{n},\textbf{C}v_{n}\rangle_{L^{2}(\rho_{n})} 
        + \kappa \frac{\eta^{2}}{2}\langle v_{n},\textbf{C}v_{n}\rangle_{L^{2}(\rho_{n})},
    \end{align*}
    which concludes the proof of Proposition \ref{lemma:descentlemma}.
\end{proof}

\section{Proof of Proposition \ref{lemma:mirrordescentlemma}}
\label{proof:mirrordescentlemma}
In this section, we provide the proof of Proposition \ref{lemma:mirrordescentlemma}. First we introduce the following useful lemmas.
\begin{lem}
    \label{lemma:int_log}
    Given two functions $g:\Omega\rightarrow \mathbb{R}$ and $h:\Omega\rightarrow \mathbb{R}$. We have:
    \begin{equation*}
        \int g(\mathbold{\lambda})\log \frac{g(\mathbold{\lambda})}{h(\mathbold{\lambda})}d\mathbold{\lambda}\geq \left(\int g(\mathbold{\lambda}) d\mathbold{\lambda} \right)\log \frac{\int g(\mathbold{\lambda})d\mathbold{\lambda}}{\int h(\mathbold{\lambda})d\mathbold{\lambda}}.
    \end{equation*}
\end{lem}
We can generalize Lemma \ref{lemma:int_log} above as follows:
\begin{lem}
\label{lemma:int_f}
    Given two functions $g:\Omega\rightarrow \mathbb{R}$, $h:\Omega\rightarrow \mathbb{R}$ and $f:\mathbb{R}\rightarrow \mathbb{R}$ is a convex function. We have:
    \begin{equation*}
        \int g(\mathbold{\lambda})f\left( \frac{h(\mathbold{\lambda})}{g(\mathbold{\lambda})}\right)d\mathbold{\lambda}\geq \left(\int g(\mathbold{\lambda}) d\mathbold{\lambda} \right)f\left( \frac{\int h(\mathbold{\lambda})d\mathbold{\lambda}}{\int g(\mathbold{\lambda})d\mathbold{\lambda}}\right).
    \end{equation*}
\end{lem}
Note that Lemma \ref{lemma:int_log} can be trivially obtained by setting $f(u)=-\log u$ in Lemma \ref{lemma:int_f}.
\begin{proof}
    Let $w(\mathbold{\lambda})=g(\mathbold{\lambda})/\int_{\Omega}g(\mathbold{\lambda})d\mathbold{\lambda}$ and $u(\mathbold{\lambda})=h(\mathbold{\lambda})/g(\mathbold{\lambda})$. Then by applying the continuous Jensen inequality \citep{horvath2012refinement} for the convex function $f$, we have:
    \begin{equation*}
    \int w(\mathbold{\lambda})f(u(\mathbold{\lambda}))d\mathbold{\lambda}\geq f\left(\int w(\mathbold{\lambda})u(\mathbold{\lambda})d\mathbold{\lambda} \right),
    \end{equation*}
    which is equivalent to:
    \begin{equation*}
    \int g(\mathbold{\lambda})f\left(\frac{h(\mathbold{\lambda})}{g(\mathbold{\lambda})}\right)d\mathbold{\lambda}\geq \left( \int g(\mathbold{\lambda})d\mathbold{\lambda}\right)f\left(\int \frac{\cancel{g(\mathbold{\lambda})}}{\int g(\mathbold{\lambda})d\mathbold{\lambda}}\frac{h(\mathbold{\lambda})}{\cancel{g(\mathbold{\lambda})}} d\mathbold{\lambda} \right),
    \end{equation*}
    which concludes the proof of Lemma \ref{lemma:int_f}.
\end{proof}

\begin{lem}
\label{lemma:smoothness}
    Let $\mathcal{L}$ be defined in (\ref{eqn:negatedelbo}). For two distributions $\rho^\prime$, $\rho\in\mathcal{P}(\Omega)$, we have:
    \begin{equation}
    \label{eqn:smoothness}
        \mathcal{L}(\rho^{\prime})- \mathcal{L}(\rho)-\int \delta \mathcal{L}(\rho)(\mathbold{\lambda})\left(\rho^{\prime}(\mathbold{\lambda})-\rho(\mathbold{\lambda}) \right)d\mathbold{\lambda} \leq \text{KL}(\rho^{\prime}, \rho).
    \end{equation}
\end{lem}
\begin{proof}
    We can write $\mathcal{L}(\rho^{\prime})-\mathcal{L}(\rho)$ as follows:
    \begin{equation*}
        \mathcal{L}(\rho^{\prime})-\mathcal{L}(\rho)= \int\int \rho^{\prime}(\mathbold{\lambda})q(\textbf{z}|\mathbold{\lambda})\left[f(\textbf{z})+\log q^{\prime}(\textbf{z}) \right]d\textbf{z}d\mathbold{\lambda} - \int\int \rho(\mathbold{\lambda})q(\textbf{z}|\mathbold{\lambda})\left[f(\textbf{z})+\log q(\textbf{z}) \right]d\textbf{z}d\mathbold{\lambda},
    \end{equation*}
    where $q^{\prime}(\textbf{z})=\int\rho^{\prime}(\mathbold{\lambda})q(\textbf{z}|\mathbold{\lambda})d\mathbold{\lambda}$ and $q(\textbf{z})=\int\rho(\mathbold{\lambda})q(\textbf{z}|\mathbold{\lambda})d\mathbold{\lambda}$.

    Thus, the left-hand side (LHS) of (\ref{eqn:smoothness}) can be rewritten as follows:
    \begin{equation}
        \label{eqn:simplified}
        \text{LHS of (\ref{eqn:smoothness})}=\int\int \rho^{\prime}(\mathbold{\lambda})q(\textbf{z}|\mathbold{\lambda})\log \frac{q^{\prime}(\textbf{z})}{q(\textbf{z})}) d\textbf{z}d\mathbold{\lambda}=\int q^{\prime}(\textbf{z})\log \frac{q^{\prime}(\textbf{z})}{q(\textbf{z})}) d\textbf{z}.
    \end{equation}
    Applying Lemma \ref{lemma:int_log} by setting $g(\mathbold{\lambda})=\rho^{\prime}(\mathbold{\lambda})q(\textbf{z}|\mathbold{\lambda})$ and $h(\mathbold{\lambda})=\rho(\mathbold{\lambda})q(\textbf{z}|\mathbold{\lambda})$, we have:
    \begin{equation}
         q^{\prime}(\textbf{z})\log \frac{q^{\prime}(\textbf{z})}{q(\textbf{z})} \leq \int \rho^{\prime}(\mathbold{\lambda})q(\textbf{z}|\mathbold{\lambda})\log \frac{\rho^{\prime}(\mathbold{\lambda})q(\textbf{z}|\mathbold{\lambda})}{\rho(\mathbold{\lambda})q(\textbf{z}|\mathbold{\lambda})}d\mathbold{\lambda}=\int \rho^{\prime}(\mathbold{\lambda})q(\textbf{z}|\mathbold{\lambda})\log \frac{\rho^{\prime}(\mathbold{\lambda})}{\rho(\mathbold{\lambda})}d\mathbold{\lambda}.
    \end{equation}
    Thus, using (\ref{eqn:simplified}), we have:
    \begin{align*}
        \text{LHS of (\ref{eqn:smoothness})}=\int\int \rho^{\prime}(\mathbold{\lambda})q(\textbf{z}|\mathbold{\lambda})\log \frac{\rho^{\prime}(\mathbold{\lambda})}{\rho(\mathbold{\lambda})}d\mathbold{\lambda}d\textbf{z}=&\int \rho^{\prime}(\mathbold{\lambda})\log \frac{\rho^{\prime}(\mathbold{\lambda})}{\rho(\mathbold{\lambda})}\left(\int q(\textbf{z}|\mathbold{\lambda}) d\textbf{z}\right)d\mathbold{\lambda}\\
        =&\int \rho^{\prime}(\mathbold{\lambda})\log \frac{\rho^{\prime}(\mathbold{\lambda})}{\rho(\mathbold{\lambda})}d\mathbold{\lambda},
    \end{align*}
    which concludes the proof of Lemma \ref{lemma:smoothness}.
\end{proof}

Now we are ready for giving the proof of Proposition \ref{lemma:mirrordescentlemma}.

\begin{proof}
    For the first update of (\ref{eqn:mirrorWGD}), we obtain the following inequality by Proposition \ref{lemma:descentlemma}:
    \begin{equation}
    \label{eqn:firstpartineq}
    \mathcal{L}(\bar{\rho}_{n})-\mathcal{L}(\rho_{n})\leq -\eta \left(1-\kappa\frac{\eta}{2} \right)\langle v_{n},\textbf{C}v_{n} \rangle_{L^{2}(\rho_{n})}.
\end{equation}
We next consider the second update of (\ref{eqn:mirrorWGD}). As $\rho_{n+1}=\text{MD}_{\eta}(\bar{\rho}_{n}, \delta \mathcal{L}(\bar{\rho}_{n}))$, the first-order optimality condition yields:
    \begin{equation}
    \label{eqn:optcondition}
        \eta \delta \mathcal{L}(\bar{\rho}_{n}) =- \log \rho_{n+1} + \log \bar{\rho}_{n} +\text{constant}
    \end{equation}
    Thus, we have:
    \begin{align*}
        &\int \delta \mathcal{L}(\bar{\rho}_{n})(\mathbold{\lambda})(\rho(\mathbold{\lambda}) - \bar{\rho}_{n}(\mathbold{\lambda}))d\mathbold{\lambda}\\
        &=\int \delta \mathcal{L}(\bar{\rho}_{n})(\mathbold{\lambda})(\rho_{n+1}(\mathbold{\lambda}) - \bar{\rho}_{n}(\mathbold{\lambda}))d\mathbold{\lambda} + \int \delta \mathcal{L}(\bar{\rho}_{n})(\mathbold{\lambda})(\rho(\mathbold{\lambda}) - \rho_{n+1}(\mathbold{\lambda}))d\mathbold{\lambda}\\
        &=\int \delta \mathcal{L}(\bar{\rho}_{n})(\mathbold{\lambda})(\rho_{n+1}(\mathbold{\lambda}) - \bar{\rho}_{n}(\mathbold{\lambda}))d\mathbold{\lambda} + \frac{1}{\eta}\int (\log \bar{\rho}_{n}(\mathbold{\lambda})-\log \rho_{n+1}(\mathbold{\lambda}))(\rho(\mathbold{\lambda}) - \rho_{n+1}(\mathbold{\lambda}))d\mathbold{\lambda}\\
        &=\int \delta \mathcal{L}(\bar{\rho}_{n})(\mathbold{\lambda})(\rho_{n+1}(\mathbold{\lambda}) - \bar{\rho}_{n}(\mathbold{\lambda}))d\mathbold{\lambda}+\frac{1}{\eta}\left[ KL(\rho_{n+1}, \bar{\rho}_{n}) + \text{KL}(\rho, \rho_{n+1}) - \text{KL}(\rho, \bar{\rho}_{n})\right],
    \end{align*}
    where the second equality is obtained by applying (\ref{eqn:optcondition}) and the third equality is obtained by three-point identity for the KL divergence.
    Taking $\rho = \bar{\rho}_{n}$, we have:
    \begin{equation}
    \label{eqn:klinequality}
        \int \delta \mathcal{L}(\bar{\rho}_{n})(\mathbold{\lambda})(\rho_{n+1}(\mathbold{\lambda}) - \bar{\rho}_{n}(\mathbold{\lambda}))d\mathbold{\lambda}=-\frac{1}{\eta}\text{KL}(\rho_{n+1},\bar{\rho}_{n}) - \frac{1}{\eta}\text{KL}(\bar{\rho}_{n}, \rho_{n+1}).
    \end{equation}
    Using (\ref{eqn:smoothness}) and (\ref{eqn:klinequality}), we have:
    \begin{align*}
        \mathcal{L}(\rho_{n+1}) - \mathcal{L}(\bar{\rho}_{n})\leq & \int\delta \mathcal{L}(\bar{\rho}_{n})(\mathbold{\lambda})\left(\rho_{n+1}(\mathbold{\lambda})-\bar{\rho}_{n}(\mathbold{\lambda}) \right)d\mathbold{\lambda} +\text{KL}(\rho_{n+1}, \bar{\rho}_{n})\\
        =& -\left(\frac{1}{\eta}-1\right)KL(\rho^{n+1},\rho^{n}) - KL(\rho^{n},\rho^{n+1}).
    \end{align*}
    Combining with (\ref{eqn:firstpartineq}), we conclude the proof of Proposition \ref{lemma:mirrordescentlemma}.
    
\end{proof}

\section{The updates (\ref{eqn:mirrorWGD}) viewed as the \textit{preconditioned} Wasserstein-Fisher-Rao gradient flow of $\mathcal{L}$}
\label{WFRandMD}

Below, we demonstrate that the updates (\ref{eqn:mirrorWGD}) are indeed related to the Wasserstein-Fisher-Rao gradient flow of $\mathcal{L}$ in the limit as $\eta\rightarrow 0$. First, it is known in \cite[Eq 2.6]{gallouet2017jko} that the gradient flow of $\mathcal{L}(\rho_{t})$ with respect to the Fisher-Rao distance is given by:

\begin{align}
    \frac{\partial \rho_{t}(\mathbold{\lambda})}{\partial t} = - \delta \mathcal{L}(\rho_{t})(\mathbold{\lambda})\rho_{t}(\mathbold{\lambda}).
\end{align}

Second, we show that the second update of (\ref{eqn:mirrorWGD}) to update the weights of particles  can be viewed as the Fisher-Rao gradient flow of $\mathcal{L}$ as $\eta\rightarrow 0$. We consider the mirror descent update formula (see (\ref{theorem:entropymd}) and (\ref{eqn:mirrorWGD})). By taking the first variation and setting it to zero, we obtain:
\begin{align}
    \delta \mathcal{L}(\rho_{n})(\mathbold{\lambda}) + \frac{1}{\eta} \log \frac{\rho_{n+1}(\mathbold{\lambda})}{\rho_{n}(\mathbold{\lambda})}=C,
\end{align}
for some constant $C>0$.Therefore,

\begin{align}
    \rho_{n+1}(\mathbold{\lambda})=\frac{\rho_{n}(\mathbold{\lambda})\exp\left(-\eta \delta \mathcal{L}(\rho_{n})(\mathbold{\lambda}) \right)}{\int \rho_{n}(\mathbold{\lambda})\exp\left(-\eta \delta \mathcal{L}(\rho_{n})(\mathbold{\lambda}) \right) d\mathbold{\lambda}}.
\end{align}
As $\eta\rightarrow 0$, we can approximate $\exp(-\eta x)\approx 1 - \eta x + O(\eta^{2})$. Thus,

\begin{align}
    \rho_{n+1}(\mathbold{\lambda})=\frac{\rho_{n}(\mathbold{\lambda})\left[1 - \eta \delta\mathcal{L}(\rho_{n})(\mathbold{\lambda}) + O(\eta^{2})\right]}{\int \rho_{n}(\mathbold{\lambda})\left[1 - \eta \delta\mathcal{L}(\rho_{n})(\mathbold{\lambda}) + O(\eta^{2})\right] d\mathbold{\lambda}}=\frac{\rho_{n}(\mathbold{\lambda})\left[1 - \eta \delta\mathcal{L}(\rho_{n})(\mathbold{\lambda}) + O(\eta^{2})\right]}{1 - \eta \left[1 + \mathcal{L}(\rho_{n}) \right]+O(\eta^{2})},
\end{align}
where we have used $\int \rho_{n}(\mathbold{\lambda})d\mathbold{\lambda}=1$ and $\int \rho_{n}(\mathbold{\lambda})\delta \mathcal{L}(\rho_{n})(\mathbold{\lambda})d\mathbold{\lambda}=\mathcal{L}(\rho_{n})+1$. Therefore,

\begin{align}
    \rho_{n+1}(\mathbold{\lambda})=\rho_{n}(\mathbold{\lambda}) \left[1 - \eta \delta\mathcal{L}(\rho_{n})(\mathbold{\lambda}) + O(\eta^{2})\right],
\end{align}
which leads to:
\begin{align}
    \frac{\rho_{n+1}(\mathbold{\lambda})-\rho_{n}(\mathbold{\lambda})}{\eta} = \delta\mathcal{L}(\rho_{n})(\mathbold{\lambda})\rho_{n}(\mathbold{\lambda}).
\end{align}
In other words, the continuous time limit of the mirror descent update is:

\begin{align*}
    \frac{\partial \rho_{t}(\mathbold{\lambda})}{\partial t} = - \delta \mathcal{L}(\rho_{t})(\mathbold{\lambda})\rho_{t}(\mathbold{\lambda}),
\end{align*}
which is identical to the Fisher-Rao flow of $\mathcal{L}$.

In summary, the first update of (\ref{eqn:mirrorWGD}) corresponds to the preconditioned Wasserstein gradient flow, while the second one aligns with the Fisher-Rao gradient flow as $\eta\rightarrow 0$. Thus, the proposed update (\ref{eqn:mirrorWGD}) can be viewed as the discrete approximation of the preconditioned Wasserstein-Fisher-Rao flow of $\mathcal{L}$.

\section{Derivation of GFlowVI and NGFlowVI for Diagonal Gaussian Variance Inference}
\label{derivations}
In this section we derive updates for GFlowVI and NGFlowVI. For the first case $\textbf{C}=\textbf{I}$, let recall $\mathbold{\lambda}_{k,n}=\left(\mathbold{\lambda}_{k,n}, \textbf{s}_{k,n} \right)$ and mean $\mathbold{\lambda}_{k,n}$ and vector $\textbf{s}_{k,n}$ are updated as follows (see (16)):
\begin{align*}
    \mathbold{\mu}_{k,n+1}=&\mathbold{\mu}_{k,n}-\eta \nabla_{\mathbold{\mu}_{k}}\mathbb{E}_{\textbf{z}\sim q(\cdot|\mathbold{\lambda}_{k,n})}\left[ f(\textbf{z}) + \log q_{n}(\textbf{z})\right].\\
    \textbf{s}_{k,n+1}=&\textbf{s}_{k,n}-\eta \nabla_{\textbf{s}_{k}}\mathbb{E}_{\textbf{z}\sim q(\cdot|\mathbold{\lambda}_{k,n})}\left[ f(\textbf{z}) + \log q_{n}(\textbf{z})\right].
\end{align*}
We denote $G = \mathbb{E}_{\textbf{z}\sim q(\cdot|\mathbold{\lambda}_{k,n})}\left[ f(\textbf{z}) + \log q_{n}(\textbf{z})\right]$. We can estimate the gradients of $G$ with respect to $\mathbold{\mu}_{k}$ and $\textbf{s}_{k}$ as follows:
\begin{align}
\label{eqn:gradmuk}
\begin{split}
    \nabla_{\mathbold{\mu}_{k}}G =& \int\nabla_{\mathbold{\mu}_{k}}q(\textbf{z}|\mathbold{\lambda}_{k,n})\left[f(\textbf{z}) + \log q_{n}(\textbf{z}) \right]d\textbf{z}  + \int q(\textbf{z}|\mathbold{\lambda}_{k,n}) \nabla_{\mathbold{\mu}_{k}} \log q_{n}(\textbf{z})d\textbf{z}\\
    =& -\int \nabla_{\textbf{z}}q(\textbf{z}|\mathbold{\lambda}_{k,n})\left[f(\textbf{z}) + \log q_{n}(\textbf{z}) \right]d\textbf{z}  + \int q(\textbf{z}|\mathbold{\lambda}_{k,n}) \frac{q(\textbf{z}|\mathbold{\lambda}_{k,n})\nabla_{\mathbold{\mu}_{k}} \log q(\textbf{z}|\mathbold{\lambda}_{k,n})}{q_{n}(\textbf{z})}d\textbf{z}\\
    =& \int q(\textbf{z}|\mathbold{\lambda}_{k,n})\left[\nabla_{\textbf{z}}f(\textbf{z}) + \nabla_{\textbf{z}}\log q_{n}(\textbf{z}) \right]d\textbf{z}  + \int q(\textbf{z}|\mathbold{\lambda}_{k,n}) \textbf{w}_{k}(\textbf{z})\nabla_{\mathbold{\mu}_{k}} \log q(\textbf{z}|\mathbold{\lambda}_{k,n})\\
    =& \mathbb{E}_{\textbf{z}\sim q(\cdot|\mathbold{\lambda}_{k,n})}\left[\nabla_{\textbf{z}}f(\textbf{z}) + \nabla_{\textbf{z}}\log q_{n}(\textbf{z}) + \textbf{w}_{k}(\textbf{z})\nabla_{\mathbold{\mu}_{k}}\log q(\textbf{z}|\mathbold{\lambda}_{k,n})\right],
\end{split}
\end{align}
where $\textbf{w}_{k}(\textbf{z})=q(\textbf{z}|\mathbold{\lambda}_{k,n})/q_{n}(\textbf{z})$.
In the second equality, we have used the identity $\nabla_{\mathbold{\mu}}q(\textbf{z}|\mathbold{\lambda})=-\nabla_{\textbf{z}}q(\textbf{z}|\mathbold{\lambda})$ for $q$ being a Gaussian distribution; in the third equality, we have used the integration by parts for the first term and $\nabla_{\mathbold{\mu}}q(\textbf{z}|\mathbold{\lambda})=q(\textbf{z}|\mathbold{\lambda})\nabla_{\mathbold{\mu}}\log q(\textbf{z}|\mathbold{\lambda})$ for the second term.

\begin{align}
\label{eqn:gradsk}
\begin{split}
    &\nabla_{\textbf{s}_{k}}G  = -1\oslash (\textbf{s}_{k,n}\odot\textbf{s}_{k,n})\odot\int\nabla_{\mathbold{\sigma}^{2}_{k}}q(\textbf{z}|\mathbold{\lambda}_{k,n})\left[f(\textbf{z}) + \log q_{n}(\textbf{z}) \right]d\textbf{z} \\
    + & \int q(\textbf{z}|\mathbold{\lambda}_{k,n}) \nabla_{\textbf{s}_{k}} \log q_{n}(\textbf{z})d\textbf{z}\\
    =&-1\oslash (\textbf{s}_{k,n}\odot\textbf{s}_{k,n})\odot\int q(\textbf{z}|\mathbold{\lambda}_{k,n})\frac{1}{2}\texttt{diag}\left[\nabla_{\textbf{z}}^{2}f(\textbf{z}) + \nabla_{\textbf{z}}^{2}\log q_{n}(\textbf{z}) \right]d\textbf{z}\\  + &\int q(\textbf{z}|\mathbold{\lambda}_{k,n}) \frac{q(\textbf{z}|\mathbold{\lambda}_{k,n})\nabla_{\textbf{s}_{k}} \log q(\textbf{z}|\mathbold{\lambda}_{k,n})}{q_{n}(\textbf{z})}d\textbf{z}\\
    =& \mathbb{E}_{\textbf{z}\sim q(\cdot|\mathbold{\lambda}_{k,n})}
    \left[
    -\frac{1}{2}\oslash (\textbf{s}_{k,n}\odot\textbf{s}_{k,n})\odot\texttt{diag}\left[\nabla_{\textbf{z}}^{2}f(\textbf{z}) + \nabla_{\textbf{z}}^{2}\log q_{n}(\textbf{z})\right] + \textbf{w}_{k}(\textbf{z})\nabla_{\textbf{s}_{k}}\log q(\textbf{z}|\mathbold{\lambda}_{k,n})
    \right],
\end{split}
\end{align}
where we have used the change of variable formula in the first equation, the relation (\ref{eqn:relation1}) for the first term of the second equation. By drawing a sample $\textbf{z}$ from $q(\textbf{z}|\mathbold{\lambda}_{k,n})$, we have the update of GFlowVI \eqref{eqn:gflow-vi}.

Next we derive the update of NGFlowVI ($\textbf{C}=\textbf{F}^{-1}$). We consider $\mathbold{\lambda}_{k}$ to be the natural parameter of the diagonal Gaussian $q(\textbf{z}|\mathbold{\lambda}_{k})$. Specifically, the natural parameters and expectation parameters of the $k$-th Gaussian at the $n$-th iteration can be defined as follows:
\begin{align*}
&\mathbold{\lambda}^{(1)}_{k,n}=\textbf{s}_{k,n}\odot\mathbold{\mu}_{k,n}, \mathbold{\lambda}^{(2)}_{k,n}=-\frac{1}{2}\textbf{s}_{k,n},\\
&\textbf{m}^{(1)}_{k,n}=\mathbold{\mu}_{k,n},\textbf{m}^{(2)}_{k,n}=\mathbold{\mu}_{k,n}\odot\mathbold{\mu}_{k,n}+1\oslash\textbf{s}_{k,n}.
\end{align*}
Then, the natural parameters are updated as follows (by \eqref{eqn:meanparamupdate} in the main text):
\begin{align*}
    \mathbold{\lambda}_{k,n+1}=\mathbold{\lambda}_{k,n}- \eta \nabla_{\textbf{m}_{k}}G.
\end{align*}
Using the chain rule (see Appendix B.1 in \cite{khan2018fast}), we can express the gradients of $G$ with respect to expectation parameter $\textbf{m}_{k}$ in terms of the gradients with respect to $\mathbold{\mu}_{k}$ and $\mathbold{\sigma}_{k}^{2}$ as follows:
\begin{align*}
    \nabla_{\textbf{m}_{k}^{(1)}}G = \nabla_{\mathbold{\mu}_{k}}G - 2\left[\nabla_{\mathbold{\sigma}^{2}_{k}}G \right]\mathbold{\mu}_{k},
    \nabla_{\textbf{m}_{k}^{(2)}}G = \nabla_{\mathbold{\sigma}^{2}_{k}}G.
\end{align*}
By following the derivation in \cite{khan2018fast}, the natural-gradient update is simplified as follows:
\begin{align}
\label{eqn:skupdate}
    \textbf{s}_{k,n+1} =& \textbf{s}_{k,n} + 2 \eta \left[\nabla_{\mathbold{\sigma}_{k}^{2}}G \right].
\end{align}
\begin{align}
\label{eqn:mukupdate}
    \mathbold{\mu}_{k,n+1} =& \mathbold{\mu}_{k,n} - \eta \left[ \nabla_{\mathbold{\mu}_{k}}G\right]\oslash \textbf{s}_{k,n+1}.
\end{align}

Using \eqref{eqn:gradsk}, we can derive the full update for $\textbf{s}_{k,n}$ in \eqref{eqn:skupdate} as follows:
\begin{align*}
    \textbf{s}_{k,n+1} =& \textbf{s}_{k,n} - 2 \eta(\textbf{s}_{k,n}\odot \textbf{s}_{k,n})  \left[\nabla_{\textbf{s}_{k}}G \right]\\
    =& \textbf{s}_{k,n} + \mathbb{E}_{\textbf{z}\sim q(\cdot|\mathbold{\lambda}_{k,n})}
    \left[
    \eta\texttt{diag}\left[\nabla_{\textbf{z}}^{2}f(\textbf{z}) + \nabla_{\textbf{z}}^{2}\log q_{n}(\textbf{z})\right]
    -2\eta\textbf{w}_{k}(\textbf{z})(\textbf{s}_{k,n}\odot\textbf{s}_{k,n})\odot\nabla_{\textbf{s}_{k}}\log q(\textbf{z}|\mathbold{\lambda}_{k,n})
    \right]
\end{align*}
Lastly, using \eqref{eqn:gradmuk}, we can derive the full update for $\mathbold{\mu}_{k,n}$ in \eqref{eqn:mukupdate} as follows:
\begin{align*}
    \mathbold{\mu}_{k,n+1} =& \mathbold{\mu}_{k,n} - \eta \left[ \nabla_{\mathbold{\mu}_{k}}G\right]\oslash \textbf{s}_{k,n+1}\\
    =& \mathbold{\mu}_{k,n} - \eta \mathbb{E}_{\textbf{z}\sim q(\cdot|\mathbold{\lambda}_{k,n})}\left[\nabla_{\textbf{z}}f(\textbf{z}) + \nabla_{\textbf{z}}\log q_{n}(\textbf{z}) + \textbf{w}_{n}(\textbf{z})\nabla_{\mathbold{\mu}_{k}}\log q(\textbf{z}|\mathbold{\lambda}_{k,n})\right]\oslash \textbf{s}_{k,n+1}.
\end{align*}
By drawing a sample $\textbf{z}$ from $q(\textbf{z}|\mathbold{\lambda}_{k,n})$, we derive the update of NGFlowVI \eqref{eqn:ngflow-vi}.

\begin{table}[]
\centering
\caption{Illustration on effect of the MD iterates on the average prediction losses of BNNs on three datasets: 'Australian scale' (negative log-likelihood), 'Boston' and 'Concrete' (mean square error). The results compare the the average prediction losses after 1000 iterations of GFlowVI and NGFlowVI (with MD iterates) against their counterparts without MD iterates (w/o-MD), demonstrating that incorporating MD iterates enhances prediction accuracy.}

\begin{tabular}{l c c r}
\hline
\textbf{Methods}&	Australian &	Boston &	Concrete\\
\hline
GFlowVI &	0.51$\pm$0.02 &	1.73$\pm$0.28	& 1.49$\pm$0.08\\
GFlowVI-w/o-MD &	0.52$\pm$0.05	& 1.81$\pm$0.15 &	1.74$\pm$0.03\\
NGFlowVI &	0.6$\pm$0.03	& 1.71$\pm$0.09 &	1.25$\pm$0.06\\
NGFlowVI-w/o-MD &	0.72$\pm$0.05	& 1.87$\pm$0.15 &	1.34$\pm$0.11\\
\hline
\end{tabular}
\label{tab:effectofMD}
\end{table}

\section{Background on the Mirror Descent Algorithm}
\label{appendix:mdalgorithm}
We provide a brief background on the mirror descent (\textbf{MD}) algorithm for optimization. Suppose we wish to minimize a function over a domain $\mathcal{X}$, say $\min_{\textbf{x}\in\mathcal{X}}f(\textbf{x})$. When $\mathcal{X}$ is unconstrained, gradient descent is the standard algorithm to optimize $f$ by solving the following optimization problem for each step $t$:
\begin{align}
\label{eqn:gd}
     \textbf{x}_{t+1}=\argmin_{\textbf{x}\in\mathcal{X}}{\langle\nabla f(\textbf{x}_{t}),\textbf{x}\rangle + \frac{1}{2\eta_{t}}\lVert \textbf{x}-\textbf{x}_{t}\rVert^{2}_{2}}.
\end{align}
To deal with the constrained optimization problems, the Mirror Descent (\textbf{MD}) algorithm replaces $\lVert \cdot\rVert_{2}$ in (\ref{eqn:gd}) with a function $\varphi$ that reflects the geometry of the problem \citep{beck2003mirror}. The \textbf{MD} algorithm chooses $\Phi$ to be the Bregman divergence induced by a strongly convex function $\varphi:\mathcal{X}\rightarrow \mathbb{R}$ as follows: $\Phi(\textbf{x}^\prime, \textbf{x})=\varphi(\textbf{x}^\prime) - \varphi(\textbf{x}) - \langle \nabla \varphi(\textbf{x}),\textbf{x}^\prime-\textbf{x}\rangle$ for $\textbf{x}^\prime, \textbf{x}\in \mathcal{X}$. Then, the solution of (\ref{eqn:gd}) for each step becomes:
\begin{align}
\label{eqn:mdupdate}
     \textbf{x}_{t+1}=\nabla \varphi^{*} \left( \nabla \varphi(\textbf{x}_{t})-\eta_{t} \nabla f(\textbf{x}_{t}) \right),
\end{align}
where $\varphi^{*}(\textbf{y})=\sup_{\textbf{x} \in \mathcal{X}}{\langle \textbf{x},\textbf{y}\rangle -\varphi(\textbf{x})}$ is the convex conjugate of function $\varphi$ and $\nabla \varphi^{*}(\textbf{y})=(\nabla \varphi)^{-1}(\textbf{y})$ is the inverse map. Intuitively, the \textbf{MD} update (\ref{eqn:mdupdate}) is composed of three steps: 1) mapping $\textbf{x}_{t}$ to $\textbf{y}_{t}$ by $\nabla \varphi$, 2) applying the update: $\textbf{y}_{k+1}=\textbf{y}_{t} - \eta_{t}\nabla f(\textbf{x}_{t})$, and 3) mapping back through $\textbf{x}_{t+1} = \nabla \varphi^{*}(\textbf{y}_{t+1})$.

\begin{table}[]
    \centering
    \caption{Banana-shaped distribution and X-shaped mixture of Gaussians.}
        \begin{tabular}{l l r}
        \hline
        \textbf{Name} & $\pi(\textbf{z})$ & \textbf{Parameters}\\
        \hline
        Banana-shaped & $\textbf{z}=\left( v_{1}, v_{1}^{2}+v_{2}+1\right), \textbf{v}\sim \mathcal{N}(0, \Sigma)$ & $\Sigma=\left[[1,0.9],[0.9, 1]\right]/0.19$\\
        X-shaped & $0.5 \mathcal{N}(\textbf{z}|0, \Sigma_{1})+0.5\mathcal{N}(\textbf{z}|0, \Sigma_{2})$ & $\Sigma_{1}=\left[[2,1.8],[1.8, 2]\right]/0.76$,$\Sigma_{2}=\left[[2, 1.8],[1.8, 2]\right]/0.76$\\
        \hline
    \end{tabular}
    \label{tab:targets}
\end{table}

\begin{table}[]
    \centering
    \caption{The approximate KL divergence between the targets and $q$ using 1000 updates of particles, averaged over five runs.}
        \begin{tabular}{l c c c r }
        \hline
        \textbf{Targets} & \textbf{WVI} & \textbf{NGVI} & \textbf{GFlowVI} & \textbf{NGFlowVI}\\
        \hline
        Banana-shaped & 0.15$\pm$0.02 & 0.32$\pm$0.01 & 0.21$\pm$0.02 & 0.12$\pm$0.02\\
        X-shaped & 0.03$\pm$0.02 & 0.05$\pm$0.03 & 0.04$\pm$0.05 & 0.02$\pm$0.02\\
        \hline
    \end{tabular}
    \label{tab:synthetic}
\end{table}

\begin{figure*}
    \centering
    \includegraphics[width=1.0\textwidth]{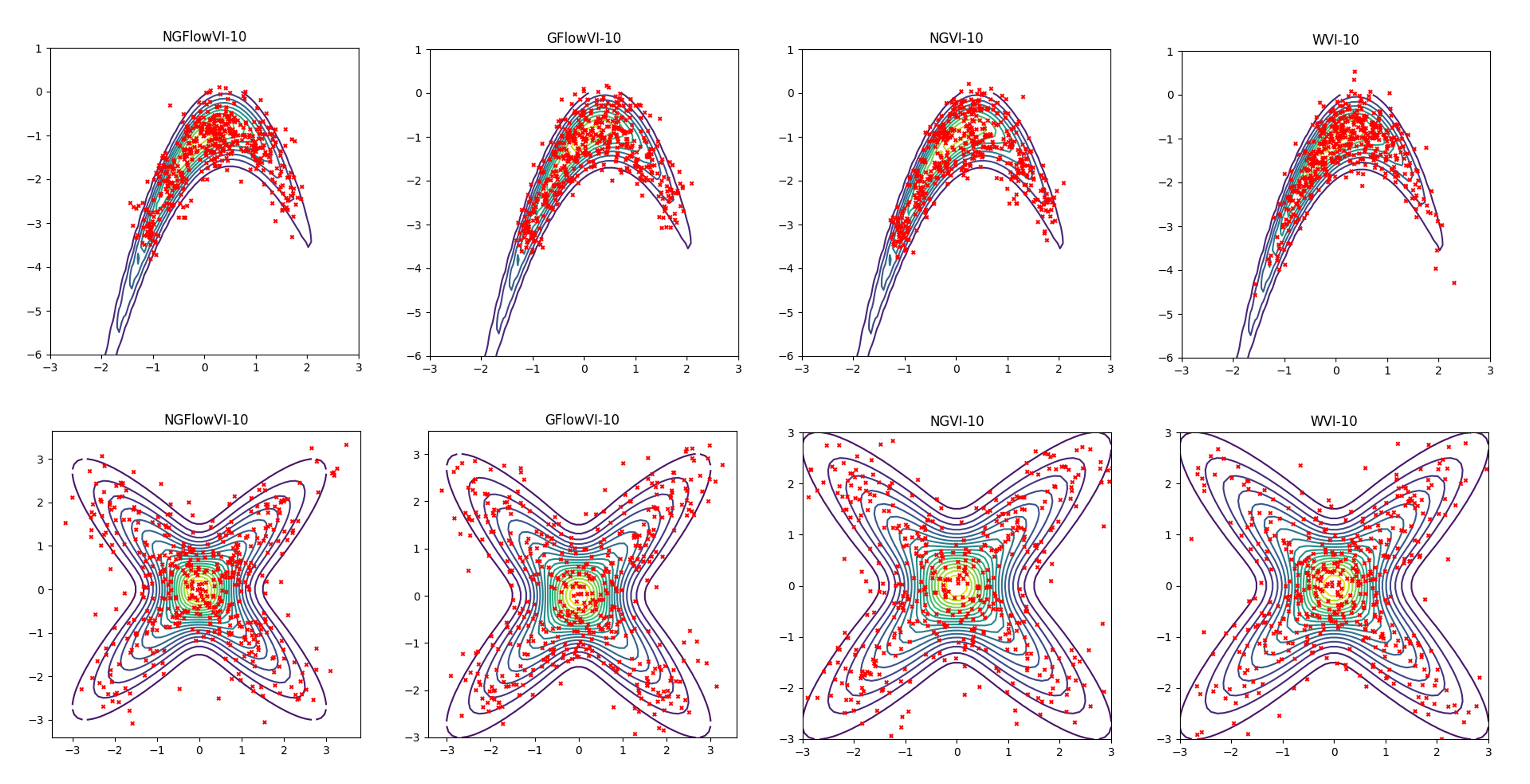}
    \caption{Experimental results on two synthetic datasets with visualization of 1000 samples from the variational distribution $q$ produced by four methods: NGFlowVI, GFlowVI, NGVI and WVI, using $K=10$ particles. These samples are used to approximate two target distributions: Banana-shaped distribution (the first row) and X-shaped distribution (the second row).}
    \label{fig:synthetic}
\end{figure*}

\section{Additional Experimental Results on Synthetic Datasets}
\label{appendix:syntheticdatasets}
In this section, we consider to sample from two other synthetic distributions defined on a two-dimensional space: a banana-shaped distribution and an X-shaped mixture of Gaussian. The densities of these distributions are given in Table \ref{tab:targets}.  

We compare WVI, NGVI, GFlowVI, and NGFlowVI, using $K=10$ particles. For these target distributions, the objective of compared methods is to produce a variational distribution $q$ that closely approximates the target distribution. Table \ref{tab:synthetic} presents the approximate KL divergence between $q$ and the target distributions, using 1000 particle updates, averaged over five runs. Figure \ref{fig:synthetic} illustrates 1000 samples from variational distributions fitted by compared methods align well with the shapes of target distributions.

\begin{table}[]
    \centering
    \caption{Average prediction losses of BNNs on three datasets: 'Australian scale' (negative log-likelihood), 'Boston' and 'Concrete' (mean square error). The best values are indicated in bold.}
        \begin{tabular}{l c c r}
        \hline
        \textbf{Methods} & Australian & Boston & Concrete \\
        \hline
        SVGD & $1.13\pm 0.04$ & $3.78\pm 0.22$ & $6.44\pm 0.92$\\
        D-Blob-CA & $1.14\pm 0.03$ & $3.24\pm 0.37$ & $4.71\pm 1.25$\\
        D-GFSD-CA & $0.98\pm 0.04$ & $3.47\pm 0.43$ & $5.22 \pm 1.02$\\
        WVI-10  & $0.85\pm 0.05$ & $3.06\pm0.17$ & $2.71\pm 0.14$ \\
        NGVI-10 & $0.62\pm 0.06$ & $2.72\pm 0.16$ & $1.49\pm 0.05$ \\
        GFlowVI-10  & $\textbf{0.51}\pm 0.02$  & $1.73\pm 0.28$ & $1.49\pm 0.08$ \\
        NGFlowVI-10 & $0.61\pm 0.03$ & $\textbf{1.71}\pm 0.09$ & $\textbf{1.25}\pm 0.06$\\
        \hline
    \end{tabular}
    \label{tab:negLoglikelihood}
\end{table}

\begin{table}[]
\centering
\caption{Average running times (in seconds) required for one epoch of methods: SVGD, WVI, NGVI, GFlowVI and NGFlowVI on three datasets: 'Australian scale', 'Boston' and 'Concrete'.}
\begin{tabular}{l c c r}
\hline
\textbf{Methods} &	Australian &	Boston	& Concrete\\
\hline
SVGD(m=50) &	3.77$\pm$0.07 &	5.23$\pm$0.02	& 10.37$\pm$0.04\\
SVGD(m=100)	& 15.19$\pm$0.14 &	20.09$\pm$0.11 &	39.93$\pm$0.33\\
SVGD(m=500) &	362.23$\pm$9.5 &	321.01$\pm$10.01	& 557.12$\pm$18.55\\
WVI	& 79.49$\pm$6.04 &	75.61$\pm$1.43 &	61.07$\pm$0.28\\
NGVI &	6.55$\pm$0.24 &	9.46$\pm$0.19 &	19.69$\pm$0.32\\
GFlowVI &	6.27$\pm$0.13 &	8.19$\pm$0.32 &	15.73$\pm$0.11\\
NGFlowVI &	6.22$\pm$0.05 &	7.98$\pm$0.17	& 16.79$\pm$0.38\\
\hline
\end{tabular}
\label{tab:runningtime}
\end{table}

\section{Additional Experimental Results on Real-world datasets}
\label{appendix:results}
\noindent
\textbf{Effects of infinite-dimensional MD}. We examined two scenarios for GFlow and NGFlow (both with K=10): one that updates both $\mathbold{\lambda}_{k}$ and $a_{k}$, for $k=1,...,K$ (denoted as GFlowVI and NGFlowVI, respectively), and another that updates $\mathbold{\lambda}_{k}$, for $k=1,...,K$, while keeping the weights fixed at 1/K (denoted by suffix w/o-MD). The results, shown in Table \ref{tab:effectofMD}, demonstrate that the weight update scheme using MD iterates improves performance.


\begin{table}[]
    \centering
    \caption{Average prediction loss on the 'Australia' dataset.}
        \begin{tabular}{l c c c c r}
        \hline
        \textbf{Methods} & $K=1$ & $K=3$ & $K=5$ & $K=10$ & $K=15$\\
        \hline
        NGFlowVI & 0.82 & 0.65 & \textbf{0.62} & \textbf{0.58} & \textbf{0.61}\\
        GFlowVI & 0.77 & 0.67 & \textbf{0.55} & \textbf{0.51} & \textbf{0.58}\\
        \hline
    \end{tabular}
    \label{tab:Australia_K}
\end{table}

\begin{table}[]
    \centering
    \caption{Average prediction loss on the 'Boston' dataset.}
        \begin{tabular}{l c c c c r}
        \hline
        \textbf{Methods} & $K=1$ & $K=3$ & $K=5$ & $K=10$ & $K=15$\\
        \hline
        NGFlowVI & 1.93 & 1.82 & \textbf{1.73} & \textbf{1.71} & \textbf{1.72}\\
        GFlowVI & 1.92 & 1.99 & \textbf{1.77} & \textbf{1.73} & \textbf{1.75}\\
        \hline
    \end{tabular}
    \label{tab:Boston_K}
\end{table}

\begin{table}[]
    \centering
    \caption{Average prediction loss on the 'Concrete' dataset.}
        \begin{tabular}{l c c c c r}
        \hline
        \textbf{Methods} & $K=1$ & $K=3$ & $K=5$ & $K=10$ & $K=15$\\
        \hline
        NGFlowVI & 1.43 & \textbf{1.26} & \textbf{1.27} & \textbf{1.25} & \textbf{1.26}\\
        GFlowVI & 1.67 & \textbf{1.53} & \textbf{1.51} & \textbf{1.49} & \textbf{1.52}\\
        \hline
    \end{tabular}
    \label{tab:Concrete_K}
\end{table}

\noindent
\textbf{Average prediction loss}. We compare our methods GFlowVI and NGFlowVI to WVI, NGVI and SVGD in terms of average prediction loss. Further, for weighting particles, we consider two DPVI algorithms \citep{zhang2021dpvi}: D-Blob-CA and DP-GFSD-CA. Like SVGD, these methods approximate the gradient of log of empirical distribution using kernels, but they dynamically adjust particle weights. We use RBF kernel and the median method \citep{liu2016stein} for SVGD, D-Blob-CA and DP-GFSD-CA, and represent $q$ with 100 samples (in latent space). We report the results after 1000 particle updates in Table \ref{tab:negLoglikelihood}. We observe that the DPVI methods perform better than SVGD due to their ability to adjust the particle weights, but still perform worse than the others. Possible reasons include the inefficiency of kernels in high-dim problems and suboptimal bandwidth selection for RBF. In addition, NGFlowVI achieves the lowest prediction errors on two datasets 'Boston' and 'Concrete', while GFlowVI achieves the lowest error on 'Australian'. Notably, they outperform the other methods in terms of the prediction loss, indicating the effectiveness of our methods.

\noindent
\textbf{Analysis on running time of methods}. We compare our methods GFlowVI and NGFlowVI, to SVGD, WVI and NGVI in terms of the computational cost. For GFlowVI, NGFlowVI, WVI and NGVI, we fix the number of components K at 10. For SVGD, we consider the number of particles of 50, 100, 500. We report the average running time (in seconds) required for one epoch of the compared methods in Table \ref{tab:runningtime}. For SVGD, the main computational cost arises from the kernel matrix, which requires  $O(m^{2})$ memory and computation for $m$ particles. In contrast, our methods and WVI, NGVI aim to update parameters of mixture components, with a memory and computational cost of $O(K)$, where K is the number of components. Thus, our methods are particle-efficient compared to SVGD. Furthermore, designing the kernel for SVGD is highly non-trivial, especially for high-dimensional problems.

\noindent
\textbf{Analysis on the mixture sizes}. To assess the impact of the mixture size (number of components K) on the performance of NGFlowVI and GFlowVI, we conducted an ablation study. We varied the mixture sizes K=1, 3, 5, 10, 15, and evaluated the average prediction losses on three datasets, as shown in Tables \ref{tab:Australia_K}, \ref{tab:Boston_K}, and \ref{tab:Concrete_K}. Each method was run for 1000 iterations per mixture size. We see that both NGFlowVI and GFlowVI achieve the highest average losses with K=1, suggesting that a single component may be insufficient to capture the posterior distribution of weights. As the mixture size increases, the losses decrease, demonstrating improved performance. Both methods remain robust with mixture sizes of 10 or more. It suggests to use cross-validation to select the number of components.

\end{document}


%

%

\onecolumn
\aistatstitle{Wasserstein Gradient Flow over Variational Parameter Space for Variational Inference: 
Supplementary Material}

\section{Proof of Theorem 1}
\begin{proof}
    To compute the first variation of $\mathcal{L}(p)$, suppose that $\varepsilon>0$ and an arbitrary distribution $\chi\in \mathcal{P}(\Omega)$. We compute $(\mathcal{L}(p+\varepsilon \chi)-\mathcal{L}(p))/\varepsilon$ as follows:
    \begin{align*}
        &\frac{1}{\varepsilon}\left[\mathcal{L}(p+\varepsilon \chi) - \mathcal{L}(p)\right]=\\
        &\frac{1}{\varepsilon}\int_{\mathbold{\lambda}}\left( p(\mathbold{\lambda})+\varepsilon \chi(\mathbold{\lambda})\right)\int_{\textbf{z}}q(\textbf{z}|\mathbold{\lambda})\left[f(\textbf{z})+\log\left( \int_{\mathbold{\lambda}}\left( p(\mathbold{\lambda})+\varepsilon \chi(\mathbold{\lambda})\right)q(\textbf{z}|\mathbold{\lambda})d\mathbold{\mathbold{\lambda}}\right)\right]d\textbf{z}d\mathbold{\lambda}\\
        -& \frac{1}{\varepsilon}\int_{\mathbold{\lambda}}p(\mathbold{\lambda})\int_{\textbf{z}}q(\textbf{z}|\mathbold{\lambda})\left[f(\textbf{z})+\log\left( \int_{\mathbold{\lambda}}p(\mathbold{\lambda})q(\textbf{z}|\mathbold{\lambda})d\mathbold{\lambda}\right)\right]d\textbf{z}d\mathbold{\lambda}\\
        =& \int_{\mathbold{\lambda}}\chi(\mathbold{\lambda})\int_{\textbf{z}}q(\textbf{z}|\mathbold{\lambda})f(\textbf{z})d\textbf{z}d\mathbold{\lambda}
        + \frac{1}{\varepsilon}\int_{\mathbold{\lambda}}\left( p(\mathbold{\lambda})+\varepsilon \chi(\mathbold{\lambda})\right)\int_{\textbf{z}}q(z|\mathbold{\lambda})\log\left( \int_{\mathbold{\lambda}}\left( p(\mathbold{\lambda})+\varepsilon \chi(\mathbold{\lambda})\right)q(\textbf{z}|\mathbold{\lambda})d\mathbold{\mathbold{\lambda}}\right)d\textbf{z}d\mathbold{\lambda}\\
        -&\frac{1}{\varepsilon}\int_{\mathbold{\lambda}}p(\mathbold{\lambda})\int_{\textbf{z}}q(\textbf{z}|\mathbold{\lambda})\log\left( \int_{\mathbold{\lambda}}p(\mathbold{\lambda})q(\textbf{z}|\mathbold{\lambda})d\mathbold{\lambda}\right)d\textbf{z}d\mathbold{\lambda}\\
        =& \int_{\mathbold{\lambda}}\chi(\mathbold{\lambda})\int_{\textbf{z}}q(\textbf{z}|\mathbold{\lambda})f(\textbf{z})d\textbf{z}d\mathbold{\lambda}\\
        +& \frac{1}{\varepsilon}\int_{\mathbold{\lambda}}\left( p(\mathbold{\lambda})+\varepsilon \chi(\mathbold{\lambda})\right)\int_{\textbf{z}}q(\textbf{z}|\mathbold{\lambda})\log\left( \int_{\mathbold{\lambda}}\left( p(\mathbold{\lambda})+\varepsilon \chi(\mathbold{\lambda})\right)q(z|\mathbold{\lambda})d\mathbold{\mathbold{\lambda}}\right)d\textbf{z}d\mathbold{\lambda} - \\
        &\frac{1}{\varepsilon}\int_{\mathbold{\lambda}}p(\mathbold{\lambda})\int_{\textbf{z}}q(\textbf{z}|\mathbold{\lambda})\log \left(\int_{\mathbold{\lambda}}\left( p(\mathbold{\lambda})+\varepsilon \chi(\mathbold{\lambda})\right)q(\textbf{z}|\lambda)d\mathbold{\lambda} \right)d\textbf{z}d\mathbold{\lambda}\\ 
        +& \frac{1}{\varepsilon}\left[ \int_{\mathbold{\lambda}}p(\mathbold{\lambda})\int_{\textbf{z}}q(\textbf{z}|\mathbold{\lambda})\log \left(\int_{\mathbold{\lambda}}\left( p(\mathbold{\lambda})+\varepsilon \chi(\mathbold{\lambda})\right)q(\textbf{z}|\mathbold{\lambda})d\mathbold{\lambda} \right)d\textbf{z}d\mathbold{\lambda}-\int_{\mathbold{\lambda}}p(\mathbold{\lambda})\int_{\textbf{z}}q(\textbf{z}|\mathbold{\lambda})\log\left( \int_{\mathbold{\lambda}}p(\mathbold{\lambda})q(z|\mathbold{\lambda})d\mathbold{\lambda}\right)d\textbf{z}d\mathbold{\lambda}\right]\\
        =& \int_{\mathbold{\lambda}}\chi(\mathbold{\lambda})\int_{\textbf{z}}q(\textbf{z}|\mathbold{\lambda})f(\textbf{z})d\textbf{z}d\mathbold{\lambda}\\
        +& \underbrace{\int_{\mathbold{\lambda}}\chi(\mathbold{\lambda})\int_{\textbf{z}}q(\textbf{z}|\mathbold{\lambda})\log\left(\int_{\mathbold{\lambda}} \left(p(\mathbold{\lambda}) + \varepsilon\chi(\mathbold{\lambda})\right)q(\textbf{z}|\mathbold{\lambda})d\mathbold{\lambda}\right)d\textbf{z}d\mathbold{\lambda}}_{\text{(a)}}\\
        +& \underbrace{\frac{1}{\varepsilon}\left[ \int_{\lambda}p(\mathbold{\lambda})\int_{\textbf{z}}q(\textbf{z}|\mathbold{\lambda})
        \log \left(1 + \frac{\varepsilon \int_{\mathbold{\lambda}}\chi(\mathbold{\lambda})q(\textbf{z}|\mathbold{\lambda})d\mathbold{\lambda}}{\int_{\mathbold{\lambda}}p(\mathbold{\lambda})q(\textbf{z}|\mathbold{\lambda})d\mathbold{\lambda}} \right)d\textbf{z}d\mathbold{\lambda}
        \right]}_{\text{(b)}}
    \end{align*}
    We process parts (a) and (b), when $\varepsilon\rightarrow 0$, as follows:
    \begin{align*}
        \lim_{\varepsilon\rightarrow 0}\text{(a)} =& \int_{\mathbold{\lambda}}\chi(\mathbold{\lambda})\int_{\textbf{z}}q(\textbf{z}|\mathbold{\lambda})\log\left(\int_{\mathbold{\lambda}} p(\mathbold{\lambda})q(\textbf{z}|\mathbold{\lambda})d\mathbold{\lambda}\right)d\textbf{z}d\mathbold{\lambda}\\
        \lim_{\varepsilon\rightarrow 0}\text{(b)} =& \int_{\lambda}p(\mathbold{\lambda})\int_{\textbf{z}}q(\textbf{z}|\mathbold{\lambda})
       \frac{\int_{\mathbold{\lambda}}\chi(\mathbold{\lambda})q(\textbf{z}|\mathbold{\lambda})d\mathbold{\lambda}}{\int_{\mathbold{\lambda}}p(\mathbold{\lambda})q(\textbf{z}|\mathbold{\lambda})d\mathbold{\lambda}}d\textbf{z}d\mathbold{\lambda}       =\int_{\textbf{z}}\int_{\mathbold{\lambda}}p(\lambda)q(\textbf{z}|\lambda)d\mathbold{\lambda}\frac{\int_{\mathbold{\lambda}}\chi(\mathbold{\lambda})q(\textbf{z}|\mathbold{\lambda})d\mathbold{\lambda}}{\int_{\mathbold{\lambda}}p(\mathbold{\lambda})q(\textbf{z}|\mathbold{\lambda})d\mathbold{\lambda}}d\textbf{z}\\
       =& \int_{\mathbold{\lambda}}\chi(\mathbold{\lambda})\int_{\textbf{z}}q(\textbf{z}|\mathbold{\lambda})d\textbf{z}d\mathbold{\lambda}= \int_{\mathbold{\lambda}}\chi(\mathbold{\lambda})d\mathbold{\lambda}
    \end{align*}
    where we have used the following equality for (b): $\lim_{\varepsilon\rightarrow0}\frac{\log \left(1 + \varepsilon x \right)}{\varepsilon}=x$ for all $x\in \mathbb{R}$.
    
    So, we have:
    \begin{align*}
        \lim_{\varepsilon\rightarrow 0}\frac{1}{\varepsilon}\left[\mathcal{L}(p+\varepsilon \chi) - \mathcal{L}(p)\right]=\int_{\mathbold{\lambda}}\chi(\mathbold{\lambda})\left(\mathbb{E}_{\textbf{z}\sim q(\cdot|\mathbold{\lambda})}\left[ f(\textbf{z} + \log q(\textbf{z})) \right] + 1\right)d\mathbold{\lambda}
    \end{align*}
    By definition of the first variation of $\mathcal{L}$, this completes the proof.
\end{proof}

\section{Proof of Proposition 3}
\begin{proof}
    Denote $\textbf{v}_{n}(\mathbold{\lambda})=-\nabla_{\mathbold{\lambda}}\mathbb{E}_{\textbf{z}\sim q(\cdot|\mathbold{\lambda})}\left[ f(\textbf{z}) + \log q_{n}(\textbf{z}) \right]$ where $q_{n}(\textbf{z})=\mathbb{E}_{\mathbold{\lambda}\sim p_{n}}\left[q(\textbf{z}|\mathbold{\lambda})\right]$, $\Phi_{t}(\mathbold{\lambda})=\mathbold{\lambda} + t \textbf{C}(\mathbold{\lambda})\textbf{v}_{n}(\mathbold{\lambda})$ for $t\in [0,\eta]$, and $\rho_{t}=(\Phi_{t})_{\#}p_{n}$. Then it is evident that $\rho_{0}=p_{n}$ and $\rho_{\eta}=p_{n+1}$.

    We define $\phi(t)=\mathcal{L}(\rho_{t})$. Clearly, $\phi(0)=\mathcal{L}(p_{n})$ and $\phi(\eta)=\mathcal{L}(p_{n+1})$. Using a Taylor expansion, we have:
    \begin{equation}
    \label{eqn:taylor}
        \phi(\eta)=\phi(0) + \eta \phi^\prime(0) + \int_{0}^{\eta}(\eta - t)\phi^{\prime\prime}(t)dt
    \end{equation}
    Using the chain rule, we can estimate $\phi^\prime(t)$ as follows:
    \begin{align*}
        \phi^\prime(t)=& \frac{d}{dt}\int_{\mathbold{\lambda}}\rho_{t}(\mathbold{\lambda})\int_{\textbf{z}}q(\textbf{z}|\mathbold{\lambda})\left[ f(\textbf{z})+\log \left( \int_{\mathbold{\lambda}}\rho_{t}(\mathbold{\lambda})q(\textbf{z}|\mathbold{\lambda})\right)\right]d\textbf{z}d\mathbold{\lambda}\\
        =& \frac{d}{dt}\int_{\mathbold{\lambda}}p_{n}(\mathbold{\lambda})\int_{\textbf{z}}q(\textbf{z}|\Phi_{t}(\mathbold{\lambda}))\left[ f(\textbf{z})+\log \left( \int_{\mathbold{\lambda}}p_{n}(\mathbold{\lambda})q(\textbf{z}|\Phi_{t}(\mathbold{\lambda}))\right)\right]d\textbf{z}d\mathbold{\lambda}\\
        =& \int_{\mathbold{\lambda}}p_{n}(\mathbold{\lambda})\langle \frac{d\Phi_{t}(\mathbold{\lambda})}{dt}, \nabla_{\mathbold{\lambda}}\int_{\textbf{z}}q(\textbf{z}|\Phi_{t}(\lambda))\left[f(\textbf{z})+\log \left(\int_{\lambda}p_{n}(\mathbold{\lambda})q(\textbf{z}|\Phi_{t}(\mathbold{\lambda})) \right) \right]d\textbf{z} \rangle d\mathbold{\lambda}\\
        =& \int_{\mathbold{\lambda}}p_{n}(\mathbold{\lambda})\langle \textbf{C}(\mathbold{\lambda})\textbf{v}_{n}(\mathbold{\lambda}), \nabla_{\mathbold{\lambda}}\mathbb{E}_{\textbf{z}\sim q(\cdot|\Phi_{t}(\mathbold{\lambda}))}\left[ f(\textbf{z}) + \log q_{t}(\textbf{z}) \right] \rangle d\mathbold{\lambda},
    \end{align*}
    where $q_{t}(\textbf{z})=\int_{\mathbold{\lambda}}p_{n}(\mathbold{\lambda})q(\textbf{z}|\Phi_{t}(\mathbold{\lambda}))d\mathbold{\lambda}$. The second equality is obtained by applying the change of variable formula, and the last equality is obtained by the definition of $\textbf{v}_{n}$ and $\Phi_{t}$.

    So, at $t=0$, we have:
    \begin{equation}
     \label{eqn:grad}
        \phi^\prime(0)= -\int_{\mathbold{\lambda}}\langle \textbf{v}_{n}(\mathbold{\lambda}), \textbf{C}(\mathbold{\lambda})\textbf{v}_{n}(\mathbold{\lambda})\rangle p_{n}(\mathbold{\lambda})d\mathbold{\lambda}=-\langle \textbf{v}_{n}, \textbf{C}\textbf{v}_{n} \rangle_{L^{2}(p_{n})}
    \end{equation}

    Next we estimate $\phi^{\prime\prime}(\mathbold{\lambda})$ as follows:
    \begin{align*}
        \phi^{\prime\prime}(\mathbold{\lambda})=& \frac{d}{dt} \phi^{\prime}(t)=\int_{\mathbold{\lambda}}p_{n}(\mathbold{\lambda})\langle \textbf{C}(\mathbold{\lambda})\textbf{v}_{n}(\mathbold{\lambda}), \frac{d}{dt}\nabla_{\mathbold{\lambda}}\mathbb{E}_{\textbf{z}\sim q(\cdot|\Phi_{t}(\mathbold{\lambda}))}\left[ f(\textbf{z}) + \log q_{t}(\textbf{z}) \right]\rangle d\mathbold{\lambda}\\
        =& \int_{\mathbold{\lambda}}p_{n}(\mathbold{\lambda})\langle \textbf{C}(\mathbold{\lambda})\textbf{v}_{n}(\mathbold{\lambda}), \textbf{H}_{t}(\mathbold{\lambda})\textbf{v}_{n}(\mathbold{\lambda})\rangle d\mathbold{\lambda} =\langle \textbf{C}\textbf{v}_{n}, \textbf{H}_{t}\textbf{v}_{n} \rangle_{L^{2}(p_{n})}, 
    \end{align*}
    where $\textbf{H}_{t}(\mathbold{\lambda})=\nabla^{2}_{\mathbold{\lambda}}\mathbb{E}_{\textbf{z}\sim q(\cdot|\Phi_{t}(\mathbold{\lambda}))}\left[  f(\textbf{z}) + \log q_{t}(\textbf{z})\right]$. Now we need to upper-bound the operator norm of $\textbf{H}_{t}$.

    Denote $\mathbold{\lambda}_{t}=\Phi_{t}(\mathbold{\lambda})$, we can rewrite 
    $\textbf{H}_{t}$ as follows:
    \begin{align*}
        &\textbf{H}_{t}(\mathbold{\lambda}) = \nabla_{\mathbold{\lambda}}\left[\int_{\textbf{z}}\nabla_{\mathbold{\lambda}}q(\textbf{z}|\mathbold{\lambda}_{t})\left[ f(\textbf{z})+\log q_{t}(\textbf{z})\right]d\textbf{z} + \int_{\textbf{z}} q(\textbf{z}|\mathbold{\lambda}_{t})\nabla_{\mathbold{\lambda}}\log q_{t}(\textbf{z})d\textbf{z}\right]\\
        =& \int_{\mathbold{\textbf{z}}}\nabla_{\mathbold{\lambda}}^{2}q(\textbf{z}|\mathbold{\lambda}_{t})\left[ f(\textbf{z})+\log q_{t}(\textbf{z})\right]d\textbf{z}
        + 2 \int_{\textbf{z}}\left[\nabla_{\mathbold{\lambda}}q(\textbf{z}|\mathbold{\lambda}_{t})\right]\left[\nabla_{\mathbold{\lambda}}\log q_{t}(\textbf{z})\right]^{\intercal} d\textbf{z} + \int_{\textbf{z}}q(\textbf{z}|\mathbold{\lambda}_{t})\nabla_{\mathbold{\lambda}}^{2}\log q_{t}(\textbf{z})d\textbf{z}\\
         =& \underbrace{\int_{\mathbold{\textbf{z}}}\nabla_{\mathbold{\lambda}}^{2}q(\textbf{z}|\mathbold{\lambda}_{t})\left[ f(\textbf{z})+\log q_{t}(\textbf{z})\right]d\textbf{z}}_{(a)}
        + 2 \underbrace{\int_{\textbf{z}}q(\textbf{z}|\mathbold{\lambda}_{t})\left[\nabla_{\mathbold{\lambda}}\log q(\textbf{z}|\mathbold{\lambda}_{t})\right]\left[\nabla_{\mathbold{\lambda}}\log q_{t}(\textbf{z})\right]^{\intercal} d\textbf{z}}_{(b)} + \underbrace{\int_{\textbf{z}}q(\textbf{z}|\mathbold{\lambda}_{t})\nabla_{\mathbold{\lambda}}^{2}\log q_{t}(\textbf{z})d\textbf{z}}_{(c)},
    \end{align*}
    where the third equality is obtained using the following identity: $\nabla_{\mathbold{\lambda}}q(\textbf{z}|\mathbold{\lambda})=q(\textbf{z}|\mathbold{\lambda})\nabla_{\mathbold{\lambda}}\log q(\textbf{z}|\mathbold{\lambda})$. We process part (a) as follows:
    \begin{align*}
        &\text{(a)} = \int_{\mathbold{\lambda}}\nabla_{\mathbold{\lambda}}\left[ q(\textbf{z}|\mathbold{\lambda}_{t})\nabla_{\lambda}\log q(\textbf{z}|\mathbold{\lambda}_{t}) \right]\left[ f(\textbf{z}) + \log q_{t}(\textbf{z}) \right]d\textbf{z}\\
        =& \int_{\textbf{z}}\nabla_{\mathbold{\lambda}} q(\textbf{z}|\mathbold{\lambda}_{t})\nabla_{\mathbold{\lambda}}\log q(\textbf{z}|\mathbold{\lambda}_{t})\left[ f(\textbf{z}) + \log q_{t}(\textbf{z}) \right]d\textbf{z}
        + \int_{\textbf{z}}q(\textbf{z}|\mathbold{\lambda}_{t})\nabla_{\lambda}^{2}\log q(\textbf{z}|\mathbold{\lambda}_{t})\left[f(\textbf{z}) + \log q_{t}(\textbf{z}) \right]d\textbf{z}\\
        =& \int_{\textbf{z}}\left[ \nabla_{\mathbold{\lambda}} q(\textbf{z}|\mathbold{\lambda}_{t})\right] \left[ \nabla_{\mathbold{\lambda}}\log q(\textbf{z}|\mathbold{\lambda}_{t})\right]^{\intercal}\left[ f(\textbf{z}) + \log q_{t}(\textbf{z}) \right]d\textbf{z}
        + \int_{\textbf{z}}q(\textbf{z}|\mathbold{\lambda}_{t})\nabla_{\lambda}^{2}\log q(\textbf{z}|\mathbold{\lambda}_{t})\left[f(\textbf{z}) + \log q_{t}(\textbf{z}) \right]d\textbf{z}\\
        =& \int_{\textbf{z}}q(\textbf{z}|\mathbold{\lambda}_{t})\left[ \nabla_{\mathbold{\lambda}} \log q(\textbf{z}|\mathbold{\lambda}_{t})\right] \left[ \nabla_{\mathbold{\lambda}}\log q(\textbf{z}|\mathbold{\lambda}_{t})\right]^{\intercal}\left[ f(\textbf{z}) + \log q_{t}(\textbf{z}) \right]d\textbf{z}
        + \int_{\textbf{z}}q(\textbf{z}|\mathbold{\lambda}_{t})\nabla_{\lambda}^{2}\log q(\textbf{z}|\mathbold{\lambda}_{t})\left[f(\textbf{z}) + \log q_{t}(\textbf{z}) \right]d\textbf{z}
    \end{align*}
    The operator norm of part (a) can be upper-bounded as follows:
    \begin{align}
    \label{eqn:parta}
    \begin{split}
    &\lVert \text{(a)}\rVert_{\text{op}} \leq  \mathbb{E}_{\textbf{z}\sim q(\cdot|\mathbold{\lambda}_{t})}\lVert \nabla_{\mathbold{\lambda}}\log q(\textbf{z}|\mathbold{\lambda}_{t})\rVert_{2}^{2}|f(\textbf{z}) + \log q_{t}(\textbf{z})| + \mathbb{E}_{\textbf{z}\sim q(\cdot|\mathbold{\lambda}_{t})}\lVert \nabla_{\mathbold{\lambda}}^{2}\log q(\textbf{z}|\mathbold{\lambda}_{t}) \rVert_{\text{op}}|f(\textbf{z})+\log q_{t}(\textbf{z})|\\
    \leq & \left[\mathbb{E}_{\textbf{z}\sim q(\cdot|\mathbold{\lambda}_{t})}\lVert \nabla_{\mathbold{\lambda}}\log q(\textbf{z}|\mathbold{\lambda}_{t})\rVert_{2}^{2}\right] \left[\mathbb{E}_{\textbf{z}\sim q(\cdot|\mathbold{\lambda}_{t})}|f(\textbf{z})| + \mathbb{E}_{\textbf{z}\sim q(\cdot|\mathbold{\lambda}_{t})}|\log q_{t}(\textbf{z})| \right]\\
    +&  \left[ \mathbb{E}_{\textbf{z}\sim q(\cdot|\mathbold{\lambda}_{t})}\lVert \nabla_{\mathbold{\lambda}}^{2}\log q(\textbf{z}|\mathbold{\lambda}_{t}) \rVert_{\text{op}}\right] \left[ \mathbb{E}_{\textbf{z}\sim q(\cdot|\mathbold{\lambda}_{t})} |f(\textbf{z})|+\mathbb{E}_{\textbf{z}\sim q(\cdot|\mathbold{\lambda}_{t})}|\log q_{t}(\textbf{z})|\right]\\
    \leq & (\alpha_{1}+\beta_{2})(M_{1}+M_{2}),
    \end{split}
    \end{align}
    where the first inequality is obtained using the equality $\lVert \textbf{a}\textbf{b}^{\intercal}\rVert_{\text{op}}= \lVert \textbf{a} \rVert_{2}\lVert \textbf{b} \rVert_{2}$ for two vectors $\textbf{a}$ and $\textbf{b}$; the second inequality is obtained by using the inequality $\mathbb{E}_{\textbf{z}\sim q(\cdot|\mathbold{\lambda})}|h_{1}(\textbf{z})||h_{2}(\textbf{z})|\leq \mathbb{E}_{\textbf{z}\sim q(\cdot|\mathbold{\lambda})}|h_{1}(\textbf{z})|\mathbb{E}_{\textbf{z}\sim q(\cdot|\mathbold{\lambda})}|h_{2}(\textbf{z})|$ for two scalar functions $h_{1}$ and $h_{2}$; the last inequality is obtained by using assumptions \textbf{(A1)}, \textbf{(A2)} and \textbf{(A3)}.

    Next, the operator norm of part (b) can be upper-bounded as follows:
    \begin{align}
    \label{eqn:partb}
    \begin{split}
        \lVert \text{(b)}\rVert_{\text{op}} = 2 \mathbb{E}_{\textbf{z}\sim q(\cdot|\mathbold{\lambda}_{t})}\lVert \nabla_{\mathbold{\lambda}}\log q(\textbf{z}|\mathbold{\lambda}_{t}) \rVert_{2}  \lVert \nabla_{\mathbold{\lambda}}\log q_{t}(\textbf{z}) \rVert_{2} \leq &\mathbb{E}_{\textbf{z}\sim q(\cdot|\mathbold{\lambda}_{t})}\lVert \nabla_{\mathbold{\lambda}}\log q(\textbf{z}|\mathbold{\lambda}_{t}) \rVert_{2}^{2} + \mathbb{E}_{\textbf{z}\sim q(\cdot|\mathbold{\lambda}_{t})}\lVert \nabla_{\mathbold{\lambda}}\log q(\textbf{z}|\mathbold{\lambda}_{t}) \rVert_{2}^{2}\\
        \leq& \alpha_{1} + \alpha_{2},
    \end{split}
    \end{align}
    where the last inequality is obtained by using assumption \textbf{(A1)}.

    Putting it all together \eqref{eqn:parta}, \eqref{eqn:partb} and assumption \textbf{A(2)} (for upper-bounding part (c)), the operator norm of $\textbf{H}_{t}$ can be upper-bounded as follows:
    \begin{align}
    \label{eqn:hessian}
        \lVert \textbf{H}_{t}(\mathbold{\lambda}) \rVert_{\text{op}} \leq (\alpha_{1} + \beta_{1})(M_{1}+M_{2}) + \alpha_{1} + \alpha_{2} +\beta_{2}=\kappa
    \end{align}
    Thus, plugging the results \eqref{eqn:grad} and \eqref{eqn:hessian} into \eqref{eqn:taylor}, we can derive the following inequality:
    \begin{align*}
        \mathcal{L}(p_{n+1}) \leq & \mathcal{L}(p_{n}) -\eta \langle\textbf{v}_{n},\textbf{C}\textbf{v}_{n}\rangle_{L^{2}(p_{n})} 
        + \int_{0}^{\eta}(\eta - t)\kappa \langle\textbf{v}_{n},\textbf{C}\textbf{v}_{n}\rangle_{L^{2}(p_{n})}dt\\
        =& \mathcal{L}(p_{n}) -\eta \langle\textbf{v}_{n},\textbf{C}\textbf{v}_{n}\rangle_{L^{2}(p_{n})} 
        + \kappa \frac{\eta^{2}}{2}\langle\textbf{v}_{n},\textbf{C}\textbf{v}_{n}\rangle_{L^{2}(p_{n})}dt 
    \end{align*}
    This completes the proof.
\end{proof}

\section{Derivation of GFlow-VI and NGFlow-VI for Gaussian mean-field Variance Inference}
In this section we derive updates for GFlow-VI and NGFlow-VI. For the fist case $\textbf{C}=\textbf{I}$, let recall $\mathbold{\lambda}_{k,n}=\left(\mathbold{\lambda}_{k,n}, \textbf{s}_{k,n} \right)$ and mean $\mathbold{\lambda}_{k,n}$ and vector $\textbf{s}_{k,n}$ are updated as follows (see (16)):
\begin{align*}
    \mathbold{\mu}_{k,n+1}=&\mathbold{\mu}_{k,n}-\eta \nabla_{\mathbold{\mu}_{k}}\mathbb{E}_{\textbf{z}\sim q(\cdot|\mathbold{\lambda}_{k,n})}\left[ f(\textbf{z}) + \log q_{n}(\textbf{z})\right]\\
    \textbf{s}_{k,n+1}=&\textbf{s}_{k,n}-\eta \nabla_{\textbf{s}_{k}}\mathbb{E}_{\textbf{z}\sim q(\cdot|\mathbold{\lambda}_{k,n})}\left[ f(\textbf{z}) + \log q_{n}(\textbf{z})\right]
\end{align*}
We denote $G = \mathbb{E}_{\textbf{z}\sim q(\cdot|\mathbold{\lambda}_{k,n})}\left[ f(\textbf{z}) + \log q_{n}(\textbf{z})\right]$. We can estimate the gradients of $G$ with respect to $\mathbold{\mu}_{k}$ and $\textbf{s}_{k}$ as follows:
\begin{align}
\label{eqn:gradmuk}
\begin{split}
    \nabla_{\mathbold{\mu}_{k}}G =& \int_{\textbf{z}}\nabla_{\mathbold{\mu}_{k}}q(\textbf{z}|\mathbold{\lambda}_{k,n})\left[f(\textbf{z}) + \log q_{n}(\textbf{z}) \right]d\textbf{z}  + \int_{\textbf{z}}q(\textbf{z}|\mathbold{\lambda}_{k,n}) \nabla_{\mathbold{\mu}_{k}} \log q_{n}(\textbf{z})d\textbf{z}\\
    =& -\int_{\textbf{z}}\nabla_{\textbf{z}}q(\textbf{z}|\mathbold{\lambda}_{k,n})\left[f(\textbf{z}) + \log q_{n}(\textbf{z}) \right]d\textbf{z}  + \int_{\textbf{z}}q(\textbf{z}|\mathbold{\lambda}_{k,n}) \frac{q(\textbf{z}|\mathbold{\lambda}_{k,n})\nabla_{\mathbold{\mu}_{k}} \log q(\textbf{z}|\mathbold{\lambda}_{k,n})}{q_{n}(\textbf{z})}d\textbf{z}\\
    =& \int_{\textbf{z}}q(\textbf{z}|\mathbold{\lambda}_{k,n})\left[\nabla_{\textbf{z}}f(\textbf{z}) + \nabla_{\textbf{z}}\log q_{n}(\textbf{z}) \right]d\textbf{z}  + \int_{\textbf{z}}q(\textbf{z}|\mathbold{\lambda}_{k,n}) \textbf{w}_{k}(\textbf{z})\nabla_{\mathbold{\mu}_{k}} \log q(\textbf{z}|\mathbold{\lambda}_{k,n})\\
    =& \mathbb{E}_{\textbf{z}\sim q(\cdot|\mathbold{\lambda}_{k,n})}\left[\nabla_{\textbf{z}}f(\textbf{z}) + \nabla_{\textbf{z}}\log q_{n}(\textbf{z}) + \textbf{w}_{k}(\textbf{z})\nabla_{\mathbold{\mu}_{k}}\log q(\textbf{z}|\mathbold{\lambda}_{k,n})\right],
\end{split}
\end{align}
where $\textbf{w}_{k}(\textbf{z})=q(\textbf{z}|\mathbold{\lambda}_{k,n})/q_{n}(\textbf{z})$.
In the second equality, we have used the identity $\nabla_{\mathbold{\mu}}q(\textbf{z}|\mathbold{\lambda})=-\nabla_{\textbf{z}}q(\textbf{z}|\mathbold{\lambda})$ for $q$ being a Gaussian distribution; in the third equality, we have used the integration by parts for the first term and $\nabla_{\mathbold{\mu}}q(\textbf{z}|\mathbold{\lambda})=q(\textbf{z}|\mathbold{\lambda})\nabla_{\mathbold{\mu}}\log q(\textbf{z}|\mathbold{\lambda})$ for the second term.

\begin{align}
\label{eqn:gradsk}
\begin{split}
    \nabla_{\textbf{s}_{k}}G  =& -1\oslash (\textbf{s}_{k,n}\odot\textbf{s}_{k,n})\odot\int_{\textbf{z}}\nabla_{\mathbold{\sigma}^{2}_{k}}q(\textbf{z}|\mathbold{\lambda}_{k,n})\left[f(\textbf{z}) + \log q_{n}(\textbf{z}) \right]d\textbf{z}  + \int_{\textbf{z}}q(\textbf{z}|\mathbold{\lambda}_{k,n}) \nabla_{\textbf{s}_{k}} \log q_{n}(\textbf{z})d\textbf{z}\\
    =&-1\oslash (\textbf{s}_{k,n}\odot\textbf{s}_{k,n})\odot\int_{\textbf{z}}q(\textbf{z}|\mathbold{\lambda}_{k,n})\frac{1}{2}\texttt{diag}\left[\nabla_{\textbf{z}}^{2}f(\textbf{z}) + \nabla_{\textbf{z}}^{2}\log q_{n}(\textbf{z}) \right]d\textbf{z}  + \int_{\textbf{z}}q(\textbf{z}|\mathbold{\lambda}_{k,n}) \frac{q(\textbf{z}|\mathbold{\lambda}_{k,n})\nabla_{\textbf{s}_{k}} \log q(\textbf{z}|\mathbold{\lambda}_{k,n})}{q_{n}(\textbf{z})}d\textbf{z}\\
    =& \mathbb{E}_{\textbf{z}\sim q(\cdot|\mathbold{\lambda}_{k,n})}
    \left[
    -\frac{1}{2}\oslash (\textbf{s}_{k,n}\odot\textbf{s}_{k,n})\odot\texttt{diag}\left[\nabla_{\textbf{z}}^{2}f(\textbf{z}) + \nabla_{\textbf{z}}^{2}\log q_{n}(\textbf{z})\right] + \textbf{w}_{k}(\textbf{z})\nabla_{\textbf{s}_{k}}\log q(\textbf{z}|\mathbold{\lambda}_{k,n}),
    \right]
\end{split}
\end{align}
where we have used the change of variable formula in the first equation, the relation (18) for the first term of the second equation. By drawing a sample $\textbf{z}$ from $q(\textbf{z}|\mathbold{\lambda}_{k,n})$, we have the update of GFlow-VI ((17) in the main text).

Next we derive the update of NGFlow-VI ($\textbf{C}=\textbf{F}^{-1}$). We consider $\mathbold{\lambda}_{k}$ to be the natural parameter of the mean-field Gaussian $q(\textbf{z}|\mathbold{\lambda}_{k})$. Specifically, the natural parameters and expectation parameters of the $k$-th Gaussian at the $n$-th iteration can be defined as follows:
\begin{align*}
&\mathbold{\lambda}^{(1)}_{k,n}=\textbf{s}_{k,n}\odot\mathbold{\mu}_{k,n}, \mathbold{\lambda}^{(2)}_{k,n}=-\frac{1}{2}\textbf{s}_{k,n}\\
&\textbf{m}^{(1)}_{k,n}=\mathbold{\mu}_{k,n},\textbf{m}^{(2)}_{k,n}=\mathbold{\mu}_{k,n}\odot\mathbold{\mu}_{k,n}+1\oslash\textbf{s}_{k,n}.
\end{align*}
Then, the natural parameters are updated as follows (by (19) in the main text):
\begin{align*}
    \mathbold{\lambda}_{k,n+1}=\mathbold{\lambda}_{k,n}- \eta \nabla_{\textbf{m}_{k}}G
\end{align*}
Using the chain rule (see Appendix B.1 in [9]), we can express the gradients of $G$ with respect to expectation parameter $\textbf{m}_{k}$ in terms of the gradients with respect to $\mathbold{\mu}_{k}$ and $\mathbold{\sigma}_{k}^{2}$ as follows:
\begin{align*}
    \nabla_{\textbf{m}_{k}^{(1)}}G = \nabla_{\mathbold{\mu}_{k}}G - 2\left[\nabla_{\mathbold{\sigma}^{2}_{k}}G \right]\mathbold{\mu}_{k},
    \nabla_{\textbf{m}_{k}^{(2)}}G = \nabla_{\mathbold{\sigma}^{2}_{k}}G
\end{align*}
By following the derivation in [9], the natural-gradient update is simplified as follows:
\begin{align}
\label{eqn:skupdate}
    \textbf{s}_{k,n+1} =& \textbf{s}_{k,n} + 2 \eta \left[\nabla_{\mathbold{\sigma}_{k}^{2}}G \right].
\end{align}
\begin{align}
\label{eqn:mukupdate}
    \mathbold{\mu}_{k,n+1} =& \mathbold{\mu}_{k,n} - \eta \left[ \nabla_{\mathbold{\mu}_{k}}G\right]\oslash \textbf{s}_{k,n+1}.
\end{align}

Using \eqref{eqn:gradsk}, we can derive the full update for $\textbf{s}_{k,n}$ in \eqref{eqn:skupdate} as follows:
\begin{align*}
    \textbf{s}_{k,n+1} =& \textbf{s}_{k,n} - 2 \eta(\textbf{s}_{k,n}\odot \textbf{s}_{k,n})  \left[\nabla_{\textbf{s}_{k}}G \right]\\
    =& \textbf{s}_{k,n} + \mathbb{E}_{\textbf{z}\sim q(\cdot|\mathbold{\lambda}_{k,n})}
    \left[
    \eta\texttt{diag}\left[\nabla_{\textbf{z}}^{2}f(\textbf{z}) + \nabla_{\textbf{z}}^{2}\log q_{n}(\textbf{z})\right] -2 \eta\textbf{w}_{k}(\textbf{z})(\textbf{s}_{k,n}\odot\textbf{s}_{k,n})\odot\nabla_{\textbf{s}_{k}}\log q(\textbf{z}|\mathbold{\lambda}_{k,n})
    \right]
\end{align*}
Lastly, using \eqref{eqn:gradmuk}, we can derive the full update for $\mathbold{\mu}_{k,n}$ in \eqref{eqn:mukupdate} as follows:
\begin{align*}
    \mathbold{\mu}_{k,n+1} =& \mathbold{\mu}_{k,n} - \eta \left[ \nabla_{\mathbold{\mu}_{k}}G\right]\oslash \textbf{s}_{k,n+1}\\
    =& \mathbold{\mu}_{k,n} - \eta \mathbb{E}_{\textbf{z}\sim q(\cdot|\mathbold{\lambda}_{k,n})}\left[\nabla_{\textbf{z}}f(\textbf{z}) + \nabla_{\textbf{z}}\log q_{n}(\textbf{z}) + \textbf{w}_{n}(\textbf{z})\nabla_{\mathbold{\mu}_{k}}\log q(\textbf{z}|\mathbold{\lambda}_{k,n})\right]\oslash \textbf{s}_{k,n+1}.
\end{align*}
By drawing a sample $\textbf{z}$ from $q(\textbf{z}|\mathbold{\lambda}_{k,n})$, we have the update of NGFlow-VI ((20) in the main text).

\vfill